\documentclass[11pt]{article}

\usepackage[preprint]{acl}

\usepackage{times}
\usepackage{latexsym}

\usepackage[T1]{fontenc}

\usepackage[utf8]{inputenc}

\usepackage{microtype}

\usepackage{inconsolata}

\usepackage{graphicx}

\usepackage{amsmath}
\usepackage{booktabs}
\usepackage{enumitem}
\usepackage{hyperref}
\usepackage[capitalize, noabbrev]{cleveref}
\usepackage{subcaption}

\usepackage{algorithm}
\usepackage{algpseudocode}
\usepackage{pifont}        

\usepackage{graphicx}
\usepackage{amsmath}
\usepackage{cleveref}
\usepackage{multirow}
\usepackage{linguex}
\usepackage{sidecap}
\usepackage[utf8]{inputenc}
\usepackage[T1]{fontenc}

\usepackage{amssymb}
\usepackage{amsthm}
\newtheorem{remark}{Remark}

\usepackage{amsthm, amsmath, amssymb}
\newtheorem{proposition}{Proposition}

\newtheorem{lemma}{Lemma} 

\usepackage{lineno}

\definecolor{darkblue}{rgb}{0, 0, 0.5}
\hypersetup{colorlinks=true, citecolor=darkblue, linkcolor=darkblue, urlcolor=darkblue}

\usepackage[most]{tcolorbox}
\usepackage{enumitem}
\usepackage{alltt}
\newtcolorbox{paperprompt}[1]{
  enhanced,
  breakable,
  colback=black!2,
  colframe=black!18,
  boxrule=0.35pt,
  arc=1.2mm,
  left=1.4mm,
  right=1.4mm,
  top=1.0mm,
  bottom=1.0mm,
  title={#1},
  fonttitle=\bfseries,
  colbacktitle=black!6,
  coltitle=black,
}

\newtcolorbox{promptschema}{
  enhanced,
  breakable,
  colback=black!1,
  colframe=black!12,
  boxrule=0.3pt,
  arc=1mm,
  left=1mm,
  right=1mm,
  top=0.6mm,
  bottom=0.6mm,
}

\newcommand{\PromptRole}[1]{%
  \par\smallskip\noindent\textbf{\textsf{#1.}}\ %
}

\newcommand{\PromptField}[1]{%
  \par\smallskip\noindent\textit{#1.}\ %
}

\usepackage{amsmath,amssymb}
\usepackage{tabularx}
\usepackage{enumitem}
\usepackage[most]{tcolorbox}

\newtcolorbox{methodbox}{
  enhanced,
  colback=white,
  colframe=black!75,
  boxrule=0.6pt,
  arc=1mm,
  left=1.2mm,
  right=1.2mm,
  top=1.0mm,
  bottom=1.0mm
}

\title{AdaptFuse: Training-Free Sequential Preference Inference \\
via Externalized Bayesian Computation}

\author{First Author \\
  Affiliation / Address line 1 \\
  Affiliation / Address line 2 \\
  Affiliation / Address line 3 \\
  \texttt{email@domain} \\\And
  Second Author \\
  Affiliation / Address line 1 \\
  Affiliation / Address line 2 \\
  Affiliation / Address line 3 \\
  \texttt{email@domain} \\}
\author{
  \textbf{Fangzhou Lin}$^{a,b,c,}$\thanks{\quad Equal contribution.}$^{,}$\thanks{\quad Email: \texttt{\{fangzhoulin1,tzz\}@tamu.edu}} {\hspace{.1em}}\quad
  \textbf{Zilin Dai}$^{d,}$\footnotemark[1] {\hspace{.1em}}\quad
  \textbf{Shuo Xing}$^{a}$ {\hspace{.1em}}\quad
  \textbf{Peiran Li}$^{a}$ {\hspace{.1em}}\quad
  \vspace{.5em}\\
  \textbf{Siyuan Yang}$^{a}$ {\hspace{.1em}}\quad
  \textbf{Qianwen Ge}$^{e}$ {\hspace{.1em}}\quad
  \textbf{Kazunori Yamada}$^{c}$ {\hspace{.1em}}\quad
  \textbf{Yang Zhou}$^{a}$ {\hspace{.1em}}\quad
  \vspace{.5em}\\
  \textbf{Ziming Zhang}$^{b}$ {\hspace{.1em}}\quad
  \textbf{Haichong Zhang}$^{b}$ {\hspace{.1em}}\quad
  \textbf{Zhengzhong Tu}$^{a,}$\footnotemark[2]$^{,}$\thanks{\quad Corresponding author.} {\hspace{.1em}}\quad
  \vspace{.5em}\\
  $^a$Texas A\&M University \quad $^b$Worcester Polytechnic Institute \\
  $^c$Tohoku University \quad $^d$Harvard University \quad $^e$Georgia Institute of Technology
}

\begin{document}
\maketitle
\begin{abstract}
Large language models (LLMs) struggle to accumulate evidence across multiple rounds of user interactions, failing to update their beliefs in a manner consistent with Bayesian inference. Existing solutions require fine-tuning on sensitive user interaction data, limiting their applicability in privacy-conscious settings.
We propose \textsc{AdaptFuse}, a training-free framework that externalizes probabilistic computation entirely from the LLM, where a symbolic module maintains a Bayesian posterior over a discrete hypothesis set, while a frozen LLM contributes semantic reasoning via multi-sample Dirichlet aggregation. The two signals are combined through entropy-adaptive fusion, which automatically weights each source by its predictive 
confidence, shifting reliance from the LLM to the symbolic posterior
as evidence accumulates. We evaluate across three domains: flight recommendation, hotel recommendation, and web shopping; on Gemma~2 9B, Llama~3 8B, and Qwen~2.5 7B. \textsc{AdaptFuse} consistently outperforms both prompting baselines and fine-tuned Bayesian Teaching models on all tasks, with accuracy improving monotonically over interaction rounds. These results demonstrate that principled inference-time algorithms can substitute for fine-tuning in personalized recommendation, without storing or training on sensitive user data. 
\end{abstract}

\begin{figure*}[ht]
    \centering
    \includegraphics[width=1.\linewidth]{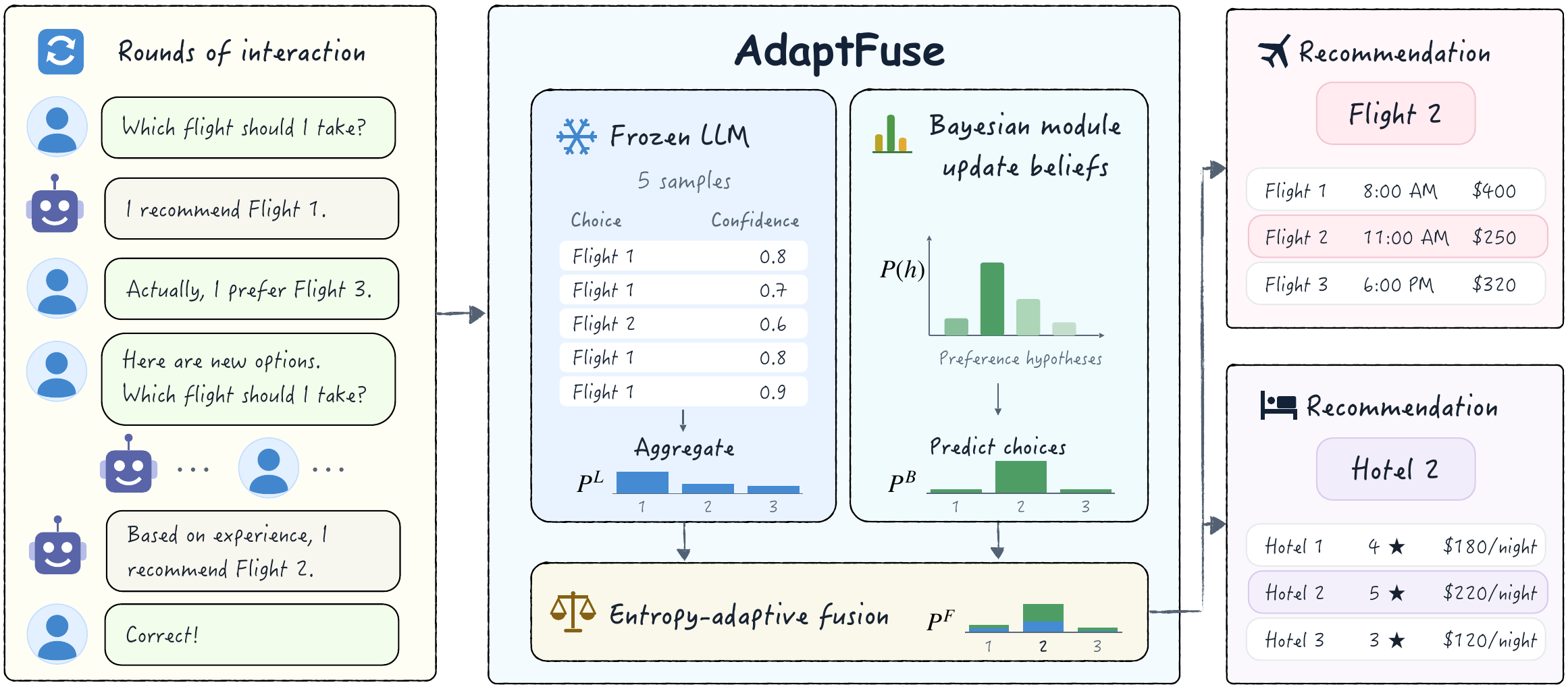}
    \vspace{-0.2in}
    \caption{Overview of AdaptFuse. Left: Sequential interactions provide feedback about user preferences. Middle: Frozen-LLM samples and self-reported confidence are aggregated into choice probabilities \(p^{L}\). The Bayesian module updates the preference posterior \(p(h)\) and derives choice probabilities \(p^{B}\). Entropy-adaptive fusion combines these into \(p^{F}\), giving greater weight to more concentrated predictions. Right: The highest-probability option is recommended, with flight and hotel panels illustrating potential applications.}
    
    \label{fig:radar}   
    \vspace{-6mm}
\end{figure*}

\section{Introduction}
\label{sec:intro}

Personalizing AI systems to individual user preferences is a fundamental challenge in interactive recommendation, dialogue systems, and e-commerce~\cite{murtaza2022ai,lindgren2025emerging,cai2025agentic,jiang2025timepre,huang2026towards}. A natural formulation is \emph{sequential preference inference}: over a series of interactions, an agent observes which items a user chooses, updates its belief about the user's latent preferences, and uses that belief to make progressively better recommendations~\cite{gao2024aligning,milani2025aligning}. The central difficulty is that user preferences are never directly 
observed; they must be inferred from a stream of noisy, indirect signals.

Large language models (LLMs) are appealing agents for this task: they possess broad semantic knowledge about items and can reason over natural language descriptions~\cite{cheng2024exploring,huang2025recommender}. However, \citet{qiu2026bayesian} recently showed that off-the-shelf LLMs fundamentally fail at sequential belief updating: 
they often plateau after a single round of feedback, failing to accumulate evidence across interactions in the way an optimal Bayesian agent would. Their proposed solution, Bayesian Teaching, achieves strong performance by
fine-tuning LLMs on demonstrations from an ideal Bayesian assistant.
Yet fine-tuning introduces a practical tension: user interaction logs contain sensitive personal preference data, and training on such data raises significant privacy concerns~\cite{gan2024navigating,yao2024survey,he2025emerged}.
Moreover, fine-tuning may need to be repeated whenever the deployment domain changes, a substantial cost barrier for real-world personalization systems. Furthermore, deploying large proprietary models (e.g., GPT series~\cite{achiam2023gpt,hurst2024gpt}) via  API introduces additional privacy risks, as user interaction data must be transmitted to external servers; lightweight open-source models that can be deployed locally offer a more privacy-preserving alternative.

In this work, we ask: \emph{can we match or exceed the performance of
fine-tuned Bayesian reasoning in LLMs without modifying any model weights?} We answer affirmatively and propose \textsc{AdaptFuse}, a training-free framework that requires no gradient 
computation, no access to model weights, and no storage of sensitive 
user data beyond what is observed at inference time.

\textsc{AdaptFuse} is built on a key insight: the bottleneck identified by \citet{qiu2026bayesian}, that LLMs struggle to verbalize and execute probabilistic belief updates, which can be circumvented entirely by externalizing the probabilistic computation. \textsc{AdaptFuse} maintains an explicit symbolic posterior over a discrete hypothesis set, updated exactly via Bayes' rule after each interaction. The LLM's role is changed to what it does well: semantic reasoning over
natural language item descriptions, elicited through $N$ diverse samples per round and aggregated via a Dirichlet-Multinomial model~\cite{minka2000estimating,madsen2005modeling,elkan2006clustering,mimno2012topic,holmes2012dirichlet}. The two signals are then combined through an entropy-adaptive fusion mechanism
that automatically weights each source by its predictive confidence,
allocating more influence to the LLM in early rounds when the symbolic posterior is diffuse, and progressively down-weighting it as evidence accumulates and the symbolic signal becomes more confident.

\textsc{AdaptFuse} is evaluated on three sequential preference inference tasks: flight recommendation~\citep{lin2022inferring}, hotel
recommendation~\citep{qiu2026bayesian}, and web
shopping~\citep{yao2022webshop}; across three open-source base models:
Gemma~2 9B~\citep{team2024gemma}, Llama~3 8B~\citep{grattafiori2024llama},
and Qwen~2.5 7B~\citep{hui2024qwen2}.
We find that \textsc{AdaptFuse} consistently outperforms both prompting baselines and fine-tuned Bayesian Teaching models~\cite{qiu2026bayesian} on all tasks and all base models, despite requiring no weight updates, training, or fine-tuning. Accuracy improves monotonically with the number of interaction rounds:
a direct consequence of Bayesian posterior concentration, whereas
prompting baselines plateau after the first round. A single instantiation of \textsc{AdaptFuse} generalizes across all three
domains without domain-specific tuning, and the symbolic module contributes negligible latency, with the per-round cost dominated by the $N$ LLM calls. These results establish that externalizing probabilistic computation is not
merely a theoretical workaround but a practically superior strategy:
\textsc{AdaptFuse} achieves stronger personalization, incurs no training cost, and never requires storing or training on sensitive user interaction data.

The paper is organized as follows.
We survey related work in \Cref{sec:related_work}, then describe the
\textsc{AdaptFuse} framework and its theoretical properties in
\Cref{sec:methods}. Experimental setup, baselines, and results are
presented in \Cref{sec:experiments}, followed by the conclusion in
\Cref{sec:conclusion}.

\section{Related Work}
\label{sec:related_work}

\textbf{Internal Probabilistic Reasoning in LLMs.}
In-context learning (ICL) capability of large language models (LLMs) enables them to perform tasks by conditioning on provided input examples, eliminating the need for explicit parameter updates~\cite{brown2020language, min2022rethinking, xie2021explanation, dai2022can, kojima2022large, wang2022self,yao2023tree, wei2022chain}. Despite these advances, probabilistic reasoning remain challenging for LLMs, prompting a growing body of work to investigate whether large language models can perform probabilistic reasoning natively~\cite{nafar2023teaching, kadavath2022language,pournemat2025reasoning}. ~\cite{gupta2025enough} suggests that LLMs can approximate Bayesian updating when provided with sufficient coin flip observations, despite a biased prior in simple tasks. More directly related to our setting, ~\cite{qiu2026bayesian} finds that LLMs fail to update their beliefs about user preferences in multi-round interactions in a manner consistent with the Bayesian framework. Our work takes a complementary route: rather than introducing Bayesian reasoning capabilities through fine-tuning, we externalize the probabilistic computation entirely.

\textbf{External Probabilistic Reasoning in LLMs.}
The dominant training-free strategy in the literature keeps LLM weights frozen while offloading probabilistic inference to an external module, BIRD~\citep{feng2024bird} uses the LLM's abductive capabilities to generate intermediate factors, which are then passed to a learnable Bayesian network for deductive probability estimation.
BRANCH~\citep{zheng2025probabilistic} factorizes joint distributions into a Bayesian network whose individual factors are estimated via LLM prompting.
GRAIL~\citep{rahimirad2025bayesian} externalizes belief tracking to a probabilistic graphical model for social deduction games while keeping the LLM frozen. MACLA~\citep{forouzandeh2025learning} maintains a frozen 7B LLM augmented with Beta-posterior-based Bayesian action selection in an external hierarchical memory. In all of these approaches, the LLM functions as a feature extractor or abductive reasoner while the Bayesian computation happens outside it. Our method differs from this paradigm in two respects: the symbolic module is \emph{free of trainable parameters} (the hypothesis set is read directly from data, with no parameter fitting required), and it supports multi-round sequential updating, which BIRD and BRANCH do not address.

\textbf{LLMs as Probabilistic Reasoning Bridges.}
A related cluster of work uses frozen LLMs as \emph{code generators} that produce probabilistic programs, which are then executed by an external inference engine.
~\cite{domke2025large} generates Stan programs~\cite{carpenter2017stan} from informal problem descriptions, runs MCMC inference on each, and combines posteriors via Bayesian model averaging weighted by marginal likelihood.
REFINESTAT~\cite{kanda2025refinestat} automates the full Bayesian modeling workflow, including LLM-generated program pruning via convergence diagnostics.
These approaches are training-free but treat the LLM purely as a compiler; the LLM itself performs no probabilistic inference. Our setting is distinct: the LLM must reason about a user's latent preferences from conversational evidence, a task that does not naturally decompose into a standalone probabilistic program.

\section{Methods}
\label{sec:methods}

We propose \textsc{AdaptFuse}, a training-free inference framework for
sequential preference inference that operates on a frozen base LLM without
any weight updates.
\textsc{AdaptFuse} maintains an explicit symbolic Bayesian posterior over a
discrete hypothesis set and combines it with LLM-derived distributional
estimates via entropy-adaptive fusion.
The core design principle is a strict separation of roles: the symbolic
module handles all probabilistic computation with mathematical guarantees,
while the frozen LLM contributes semantic reasoning about item descriptions.

\paragraph{Notation.}
$[K] \triangleq \{1,\ldots,K\}$ denotes the option index set.
$\mathcal{H} = \{h_m\}_{m=1}^{M}$ is the hypothesis set of $M$ candidate
preference vectors.
$\mathbf{b}_t \in \Delta^{M-1}$ is the belief vector \emph{after}
observing the user's choice at round $t$, with
$b_{t,m} \triangleq P(h_m \mid D_t)$, where $D_t$ collects all observed
(option set, user choice) pairs through round $t$.
We write $\mathbf{b}_0$ for the prior (before any observation).
At round $t$, the agent predicts using $\mathbf{b}_{t-1}$, observes
$y_t$, and updates to $\mathbf{b}_t$.
$\pi \in \Delta^{K-1}$ denotes a probability vector over options.
All logarithms are natural.

\subsection{Task: Sequential Preference Inference}
\label{sec:task}

We consider a general sequential preference inference setting.
An agent interacts with a user over $T$ rounds.
At each round $t$, the agent receives a set of $K$ items described in
natural language, predicts which item the user prefers, and observes the
true choice $y_t \in [K]$.
The user's latent preference is assumed fixed throughout the interaction.
After round $T$, the belief is frozen and used to predict the user's
preferred item from \emph{held-out} sets $\tilde{X}$ not seen during
interaction; held-out accuracy is our primary evaluation metric, as it
measures generalization of the inferred preference rather than adaptation
to specific seen items.

We instantiate this setting across three task domains: flight recommendation~\citep{lin2022inferring}, hotel recommendation~\citep{qiu2026bayesian}, and web shopping~\citep{yao2022webshop};
each with domain-specific item attributes described in \Cref{sec:experiments}.

\paragraph{Feature representation.}
A deterministic parser maps each item's natural language description to
$x \in [0,1]^d$, falling back to a uniform prediction if parsing fails.
Domain-specific dimensions and attribute ranges are given in
\Cref{sec:experiments}.

\paragraph{Preference model.}
The user's latent preference is modeled by the utility
$U_h(x) \triangleq h^\top x$, where $h \in \mathbb{R}^d$ is the
latent preference weight vector.

\paragraph{Hypothesis set.}
Rather than inferring $h$ continuously, we discretize them.
The hypothesis set $\mathcal{H} = \{h_1, \ldots, h_M\}$ is constructed by
extracting all unique ground-truth preference vectors from the training
split of the interaction dataset.
This yields $M$ candidate preference hypotheses as a fixed read-only lookup table,
requiring no parameter fitting.
Discretization is the key step that makes exact Bayesian inference
tractable without MCMC or variational approximation~\cite{fayyad1993multi,zaiser2023exact}.
A brief discussion of this approximation is provided in
\Cref{app:discrete_approx}.

We note that this construction guarantees $h^* \in \mathcal{H}$ for
every evaluation user only when the train/test split ensures full
coverage of all preference types.
In our experimental setup this holds by construction; however, the
framework degrades gracefully when this assumption is violated
(see \Cref{app:discrete_approx}).

\paragraph{Choice likelihood.}
Given option set $X = \{x_i\}_{i=1}^K$ and hypothesis $h$, the
probability of the user choosing option $i$ follows the Luce choice
model~\citep{luce1959individual,luce1977choice}:
\begin{equation}
    P(y = i \mid X, h)
    = \frac{\exp(\beta \cdot U_h(x_i))}
           {\sum_{j=1}^{K} \exp(\beta \cdot U_h(x_j))}
    \label{eq:likelihood}
\end{equation}
where $\beta > 0$ is the inverse temperature (default $\beta = 6.0$).

\subsection{\textsc{AdaptFuse}}
\label{sec:method}

\textsc{AdaptFuse} operates in three stages at each decision point:
\emph{(i)} symbolic Bayesian belief tracking,
\emph{(ii)} LLM multi-sample aggregation, and
\emph{(iii)} entropy-adaptive fusion.
The complete procedure is provided in Algorithm~\ref{alg:adaptfuse}
in Appendix~\ref{app:algorithm}.

\subsubsection{Stage 1: Symbolic Bayesian Belief Tracking}
\label{sec:stage1}

\paragraph{Prior and posterior update.}
The belief is initialized uniformly: $b_{0,m} = 1/M$ for all $m$.
At round $t$, let $X_t = \{x_{t,i}\}_{i=1}^{K}$ be the current option
set and $y_t$ the observed choice. For hypothesis $h_m$, $\ell_{t,m}$
denotes the likelihood of this observation, floored at $\varepsilon$.
The belief is updated as:
\begin{equation}
\begin{aligned}
    b_{t,m} &\;\propto\; b_{t-1,m} \cdot \ell_{t,m}, \\
    \ell_{t,m} &\triangleq \max\!\bigl(\varepsilon,\;
    P(y_t \mid X_t, h_m)\bigr),
\end{aligned}
\label{eq:bayes}
\end{equation}
with floor $\varepsilon = 10^{-8}$ preventing any hypothesis from being
irreversibly eliminated by a single atypical observation, and the vector
renormalized after each update.
The equivalent closed-form posterior is given in \Cref{sec:impl}.

\paragraph{Symbolic predictive distribution.}
At round $t$, the symbolic prediction over options is computed from the
\emph{pre-update} belief $\mathbf{b}_{t-1}$ as the posterior-weighted
mixture:
\begin{equation}
    \pi^{\mathrm{sym}}_i
    \;\triangleq\; \sum_{m=1}^{M} b_{t-1,m} \cdot P(y_t=i \mid X_t, h_m),
    \label{eq:predictive}
\end{equation}
where $\pi^{\mathrm{sym}} = (\pi^{\mathrm{sym}}_1,\ldots,
\pi^{\mathrm{sym}}_K)$ is the symbolic predictive probability vector
for the current option set $X_t$. Its component $\pi^{\mathrm{sym}}_i$
is the predicted probability that the user selects option $i$, with
$i\in[K]$ indexing options and $m\in\{1,\ldots,M\}$ indexing
preference hypotheses. The round index $t$ is omitted from
$\pi^{\mathrm{sym}}$ for brevity. This mixture marginalizes exactly
over hypotheses and requires no sampling.

\subsubsection{Stage 2: LLM Multi-Sample Aggregation}
\label{sec:stage2}

The frozen LLM is queried $N = 5$ times per decision point.
To encourage output diversity, each sample $s \in [N]$ uses a temperature
drawn cyclically from $\{0.2, 0.7, 1.0\}$ and a reasoning hint from a fixed
pool (see \Cref{sec:prompt} for example task interactions).
Each valid sample $s$ yields a predicted option $\widehat{y}_s \in [K]$
and a confidence score $c_s \in [0,1]$; samples that fail to parse are
silently skipped.

\paragraph{Dirichlet aggregation.}
Samples are aggregated via a Dirichlet-Multinomial model~\cite{madsen2005modeling,elkan2006clustering}.
Starting from $\alpha_i = \alpha_0 = 1.0$, each valid sample $s$ contributes
confidence-weighted pseudo-counts $w_{s, i}$:
\begin{equation}
\begin{aligned}
    \alpha_i &\;\leftarrow\; \alpha_i + w_{s,i}, \\
    w_{s,i} &\triangleq
    \begin{cases}
        c_s & i = \widehat{y}_s, \\[3pt]
        (1 - c_s)/(K-1) & \text{otherwise.}
    \end{cases}
\end{aligned}
\label{eq:dirichlet_update}
\end{equation}
The aggregated LLM distribution is the Dirichlet posterior mean,
$\pi^{\mathrm{llm}}_{\mathrm{raw},t}
= \boldsymbol{\alpha}/\|\boldsymbol{\alpha}\|_1$.

The posterior-mean interpretation is formalized in
\Cref{lemma:dirichlet}; the statement and proof appear in
\Cref{app:proof_dirichlet}.

\noindent
Unlike majority vote, Dirichlet aggregation weights confident samples more
heavily and retains positive mass on all options via the $\alpha_0$ prior.
When all samples fail to parse, the raw aggregate
$\pi^{\mathrm{llm}}_{\mathrm{raw},t}$ reduces to the uniform distribution.

\paragraph{Temporal smoothing.}
We smooth the per-round LLM distribution with an exponential moving
average (EMA), using coefficient $m=0.65$ and initializing it with the
first round's raw aggregate. The running average is updated only during
interaction and remains read-only during held-out evaluation.
The update equation and effective sample count are given in
\Cref{sec:impl}.

\subsubsection{Stage 3: Entropy-Adaptive Fusion}
\label{sec:stage3}

\textsc{AdaptFuse} weights each source by its predictive confidence,
measured via normalized entropy.
For any $\pi \in \Delta^{K-1}$, define
$\tilde{H}(\pi) = -\sum_{i=1}^K \pi_i \log \pi_i / \log K
\in [0,1]$.
The source weights are:

\begin{equation}
\begin{aligned}
    w^{\mathrm{llm}} &\triangleq \max\!\bigl(\varepsilon_w,\;
        1 - \tilde{H}(\pi^{\mathrm{llm}}_t)\bigr), \\
    w^{\mathrm{sym}} &\triangleq \max\!\bigl(\varepsilon_w,\;
        1 - \tilde{H}(\pi^{\mathrm{sym}})\bigr).
\end{aligned}
\label{eq:adaptive_weights}
\end{equation}

with $\varepsilon_w = 10^{-3}$.
The fused prediction distribution is:
\begin{equation}
    \pi^*(i) \;\propto\;
    w^{\mathrm{llm}} \cdot \pi^{\mathrm{llm}}_i
    + w^{\mathrm{sym}} \cdot \pi^{\mathrm{sym}}_i,
    \label{eq:adaptive_fusion}
\end{equation}
and the predicted option is $i^* = \operatorname{argmax}_i \pi^*(i)$.

The entropy-adaptive weighting implements an automatic information
schedule.
At any round $t$, the LLM's share of the fused prediction is bounded
above by $1/(1 + w^{\mathrm{sym}}_t)$, which decreases monotonically as
$w^{\mathrm{sym}}_t$ grows
(\Cref{prop:schedule} in \Cref{app:proof_schedule}).
In early rounds, the symbolic posterior is diffuse, both sources carry
comparable weight, and the LLM's semantic understanding of item
descriptions contributes meaningfully.
As evidence accumulates and the symbolic posterior concentrates,
$\tilde{H}(\pi^{\mathrm{sym}})$ decreases, $w^{\mathrm{sym}}$ increases,
and the LLM signal is progressively down-weighted; without any
manually tuned schedule.
We verify this behavior empirically in \Cref{sec:experiments}.

\subsubsection{Sequential Update}
\label{sec:update}

After prediction $i^*_t$ and observation of $y_t$, the belief is updated
via \Cref{eq:bayes} unconditionally at every interaction round, regardless
of prediction correctness.
Held-out evaluations do not trigger a belief update.

\section{Experiments}
\label{sec:experiments}
\subsection{Tasks and Domain Configuration}
\label{sec:tasks}
Flight~\citep{lin2022inferring} and Hotel~\citep{qiu2026bayesian} use $T=5$ interactions with $K=3$ candidates per set. Their four attributes have domain-specific meanings but share the linear preference model. After each round, accuracy is measured on 50 held-out sets from the same item distribution. Preferences are fixed, and the main split covers preference types in the training-derived support. This is generalization to unseen items within known preference support, rather than an unseen-user-type split.

\begin{table}[t]
\centering\small
\setlength{\tabcolsep}{3pt}
\begin{tabularx}{\columnwidth}{@{}l>{\raggedright\arraybackslash}Xr@{}}
\toprule
Domain & Features / representation & $M$\\
\midrule
Flight & Departure, duration, stops, price & 96\\
Hotel & Distance, price, rating, amenities & 112\\
WebShopping & Title and product description & 187\\
\bottomrule
\end{tabularx}
\vspace{-0.1in}
\caption{Domain-specific inputs and hypothesis support. Flight and Hotel use $d=4$ normalized scalars; WebShopping uses approximate extraction from free text.}
\label{tab:domain_config}
\vspace{-0.2in}
\end{table}

For Flight, departure time spans 06:00--22:00, duration 30 minutes--20 hours, stops 0--2, and price \$100--\$1000. Hotel uses distance 0.5--20 km, nightly price \$50--\$500, rating 1--5, and amenities 0--10. Each range is mapped to $[0,1]$ before applying the signed preference vector. The support is constructed separately for each domain; transferring the inference equations does not transfer these features or annotations automatically.

WebShopping uses products and goals from WebShop~\citep{yao2022webshop}. Product titles and descriptions are free text, while a goal may specify material, size, and price constraints. Preferred products are defined using the benchmark reward combining attribute match, price, and rating. The task examples reveal the preferred alternative after an incorrect recommendation. This is richer than correctness alone: \Cref{eq:bayes} requires the selected alternative. For binary-only feedback, the likelihood must instead sum over choices consistent with that feedback. We do not claim the reported results establish performance under that weaker observation protocol. The WebShopping extraction schema and mapping from free-text goals to finite support also require more specification than the structured domains; the cross-domain table should be interpreted within this constraint.

Dataset statistics are summarized in Appendix~\ref{app:dataset_statistics}; sources and license notices are documented in Appendix~\ref{app:dataset_licenses}.

\subsection{Models, Baselines, and Comparison Scope}
\label{sec:baselines}
We use Gemma 2 9B~\citep{team2024gemma}, Llama 3 8B~\citep{grattafiori2024llama}, and Qwen 2.5 7B~\citep{hui2024qwen2}, keeping their weights frozen for \textsc{AdaptFuse}. Direct prompting conditions on the interaction history. CoT~\citep{wei2022chain} adds step-by-step reasoning; self-consistency~\citep{wang2022self} uses majority vote over five CoT samples. Oracle Learning and Bayesian Teaching use fine-tuned models from \citet{qiu2026bayesian}, trained on ground-truth or Bayesian-teacher demonstrations, respectively. Their training cost is distinct from the per-round latency reported here.

The system receives both the text interface and privileged domain structure: parsed attributes, a hand-specified likelihood, and a training-derived support. The prompting baselines receive history and descriptions but are not given the same explicit posterior machinery. Thus a difference in accuracy compares complete systems, not equal-information inference algorithms. Self-consistency matches the five-call count, but does not match the feature representation or symbolic knowledge. A symbolic-only predictor $\arg\max_i\pi_i^{\mathrm{sym}}$ with the identical parser, $\mathcal H$, $\beta$, and feedback would be needed to determine whether LLM fusion improves over the available Bayesian model. Such results are not available in this evaluation, and are not inferred from the self-consistency gap.

\paragraph{Configuration and reporting.}
The common aggregation parameters are $N=5$ and $\alpha_0=1$, with fusion floor $10^{-3}$ and likelihood floor $10^{-8}$. The legacy default includes $m=0.65$ positional EMA; the no-EMA variant uses $m=0$. Hyperparameters were not tuned on held-out data. Unless explicitly labeled no EMA, the sampling, support-size, inverse-temperature, and misspecification results below use the legacy default. Seed-level predictions and confidence intervals are unavailable, so small differences are descriptive rather than claims of statistical significance.

\subsection{Flight Results and Sequential Belief Updating}
\label{sec:results}
\begin{table}[t]
\centering\footnotesize
\setlength{\tabcolsep}{2pt}
\begin{tabularx}{\columnwidth}{@{}>{\raggedright\arraybackslash}Xrrrrrr@{}}
\toprule
&\multicolumn{2}{c}{Gemma}&\multicolumn{2}{c}{Llama}&\multicolumn{2}{c}{Qwen}\\
\cmidrule(lr){2-3}\cmidrule(lr){4-5}\cmidrule(lr){6-7}
Method&First&Final&First&Final&First&Final\\
\midrule
Direct&31.3&33.5&30.2&34.6&31.2&34.1\\
CoT&36.2&38.4&35.1&39.1&34.5&38.6\\
Self-consistency&40.4&44.7&38.1&43.8&37.2&41.5\\
Oracle Learning&48.6&59.1&45.1&58.3&40.6&51.7\\
Bayesian Teaching&54.1&73.4&55.5&74.3&52.5&65.2\\
AdaptFuse + EMA&53.1&76.2&52.1&76.3&50.4&67.3\\
\bottomrule
\end{tabularx}
\vspace{-0.1in}
\caption{Flight accuracy (\%) after the first and fifth rounds. EMA denotes original positional configuration.}
\label{tab:main_results_flight}
\vspace{-0.2in}
\end{table}

\Cref{tab:main_results_flight} shows substantial differences between frozen prompting and systems using preference-specific inference or training. On Gemma, five-sample self-consistency reaches 44.7\%, Oracle Learning 59.1\%, and Bayesian Teaching 73.4\%. The legacy \textsc{AdaptFuse} result is 76.2\%. A separate no-EMA ablation was conducted only on Gemma, achieving 52.3\% after the first round and 74.0\% after the fifth round. Relative to Bayesian Teaching, the original configuration gains 2.8 points on Gemma, 2.0 on Llama, and 2.1 on Qwen. The Gemma advantage shrinks to 0.6 points without EMA. Hence much of the original margin depends on a temporal operation whose candidate alignment is not established. These results support competitive system-level performance in the covered-support benchmark, while evidence for the candidate-independent core is narrower.

\begin{figure*}[t]
\centering
\includegraphics[width=\linewidth]{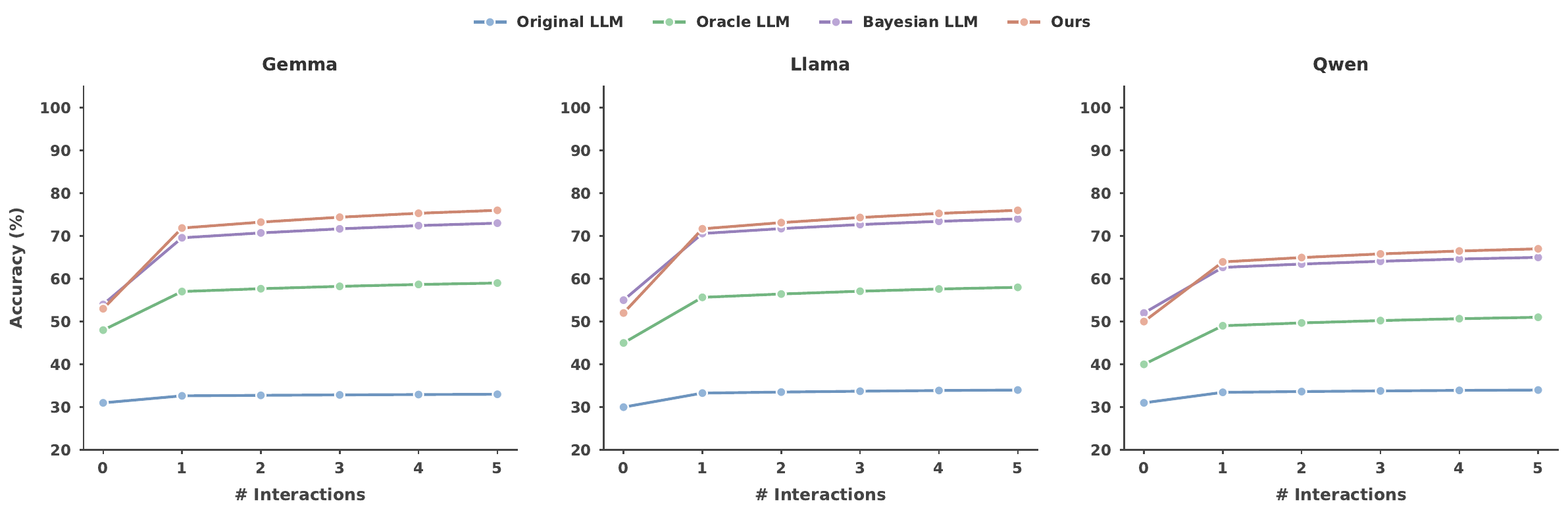}
\vspace{-7mm}
\caption{Reported Flight accuracy through five interactions. The plotted methods are Original, Oracle Learning, Bayesian Teaching, and the original \textsc{AdaptFuse} configuration with EMA. CoT, self-consistency, and no-EMA trajectories are not plotted.}
\label{fig:flight_rounds}
\vspace{-0.17in}
\end{figure*}

The plotted legacy trajectories in \Cref{fig:flight_rounds} improve after the first interaction, consistent with useful accumulation of preference evidence. Bayesian Teaching also improves; direct prompting improves much less. The first-round results favor Bayesian Teaching over \textsc{AdaptFuse} for all three models, before the ordering reverses at the final round. The observed curves do not guarantee monotonic improvement: observation noise and candidate-set difficulty can affect accuracy.

The final errors remain appreciable: 23.8\%, 23.7\%, and 32.7\% in the legacy configuration. An evaluation using the true $h^*$ in \Cref{eq:likelihood} would help separate stochastic choice noise from inference error and parser/model mismatch. That ceiling is unreported, so residual errors cannot be attributed solely to insufficient Bayesian concentration. \Cref{fig:features} varies attribute dimensionality from two to eight. Accuracy generally becomes harder to maintain as the representation grows, but \textsc{AdaptFuse} does not uniformly exceed Bayesian Teaching at every dimensionality. Changing feature count can alter identifiability, utility gaps, and support coverage together; this probe measures task sensitivity rather than transfer to an unknown feature schema.

\begin{figure*}[t]
\centering
\includegraphics[width=\linewidth]{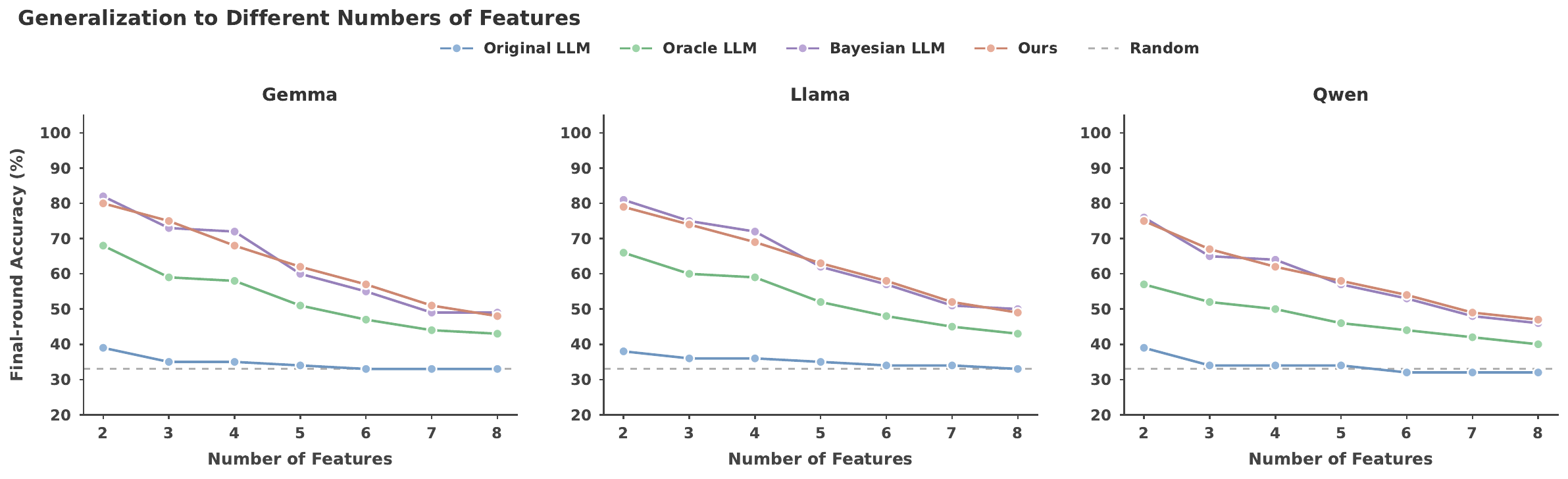}
\caption{Reported final Flight accuracy as the number of attributes varies from two to eight. ``Ours'' denotes the original EMA configuration. Relative method rankings vary across dimensionalities.}
\label{fig:features}
\vspace{-0.2in}
\end{figure*}

\subsection{Inference Budget and Component Ablations}
\label{sec:acc_latency}
\begin{table}[t]
\centering\small
\setlength{\tabcolsep}{3pt}
\begin{tabularx}{\columnwidth}{@{}>{\raggedright\arraybackslash}Xrrrr@{}}
\toprule
Method&Gemma&Llama&Qwen&Calls\\
\midrule
Direct&0.52&0.74&0.63&1\\
CoT&2.15&2.87&2.42&1\\
Self-consistency&3.67&3.79&3.68&5\\
Bayesian Teaching&1.42&1.34&1.05&1\\
AdaptFuse + EMA&4.35&4.63&4.51&5\\
\bottomrule
\end{tabularx}
\vspace{-0.1in}
\caption{Reported mean seconds per interaction round on one A100 GPU, Flight. Training cost is excluded. Calls denote LLM invocations, rather than matched generated tokens.}
\label{tab:inference_time}
\vspace{-0.2in}
\end{table}

Repeated LLM generation dominates inference cost. The default legacy configuration takes 4.35, 4.63, and 4.51 seconds per round, approximately 3.1--4.3$\times$ Bayesian Teaching's latency (\Cref{tab:inference_time}). Frozen weights avoid task-specific fine-tuning but do not reduce inference cost. Symbolic overhead is reported as small, without a separate timing breakdown.

\begin{table}[h]
\centering\small
\setlength{\tabcolsep}{3pt}
\begin{tabularx}{\columnwidth}{@{}>{\raggedright\arraybackslash}Xrr@{}}
\toprule
Configuration&Seconds&Accuracy (\%)\\
\midrule
EMA, $N=1$&0.91&71.6\\
$N=2$&1.77&72.8\\
$N=3$&2.63&73.3\\
$N=4$&3.49&74.6\\
$N=5$&4.35&76.2\\
Bayesian Teaching&1.42&73.4\\
\bottomrule
\end{tabularx}
\vspace{-0.1in}
\caption{Sampling trade-off on Flight with Gemma. Other components, including legacy EMA, are fixed. The no-EMA core's corresponding sampling curve is not measured.}
\label{tab:acc_latency_tradeoff}
\vspace{-0.2in}
\end{table}

Increasing $N$ improves the reported accuracy while latency grows approximately linearly (\Cref{tab:acc_latency_tradeoff}). One sample is faster but loses 1.8 points to Bayesian Teaching; three samples reach 73.3\% at 2.63 seconds, close in accuracy but slower. Four or five samples exceed the baseline point estimate. This offers a practical budget control, but does not establish superiority at matched wall-clock time or matched tokens. The cost of maintaining domain features and hypothesis coverage also lies outside these inference timings.

\begin{table}[h]
\vspace{-0.1in}
\centering\small
\setlength{\tabcolsep}{3pt}
\begin{tabularx}{\columnwidth}{@{}>{\raggedright\arraybackslash}Xrrr@{}}
\toprule
Variant&Flight&Hotel&WebShopping\\
\midrule
Fixed fusion&72.7&67.9&55.1\\
Majority vote&73.1&68.0&56.8\\
Dirichlet, no EMA&74.0&69.3&57.0\\
Dirichlet + EMA&76.2&70.5&57.6\\
\bottomrule
\end{tabularx}
\vspace{-0.1in}
\caption{Gemma cross-domain component ablations: final accuracy (\%). Fixed fusion and majority vote replace the named component in the original default; no EMA removes positional smoothing.}
\label{tab:cross_domain_ablation}
\label{tab:ablation_aggregation}
\vspace{-0.1in}
\end{table}

\Cref{tab:cross_domain_ablation} reports the available cross-domain component comparison. These Gemma estimates support comparison among the listed configurations. They do not provide a matched Bayesian Teaching baseline for Hotel or WebShopping, so our cross-domain conclusions concern component sensitivity rather than a verified advantage over fine-tuning.

On Flight, the default exceeds fixed fusion by 3.5 points and majority vote by 3.1. Removing EMA yields 74.0\%, a 2.2-point reduction. On Hotel, no EMA reaches 69.3\% versus 70.5\% with EMA; on WebShopping, the difference is 57.0\% versus 57.6\%. These patterns show that component effects vary with the domain, but they do not identify the cause of the EMA gain as variance reduction. Positional persistence or benchmark ordering could also contribute, and sample dependence remains unmeasured. Candidate permutation tests would distinguish an order effect from improved preference inference. The no-EMA configurations remain above fixed-fusion and majority-vote point estimates in all three domains, although those comparisons also change the smoothing setting. To isolate each mechanism cleanly, the same no-EMA backbone should be used for fixed versus adaptive fusion and for confidence weighting versus unweighted aggregation. The grid therefore measures component sensitivity, rather than a complete causal decomposition.

\subsection{Support Coverage and Likelihood Sensitivity}
\label{sec:misspecification}
The main split covers every evaluation preference type. The misspecification probe removes test types from $\mathcal H$ and defines mild and severe settings by the distance to the nearest retained vector. This changes the explanations available to the symbolic module while keeping the frozen-model approach.

\begin{table}[t]
\centering\small
\setlength{\tabcolsep}{3pt}
\begin{tabularx}{\columnwidth}{@{}>{\raggedright\arraybackslash}Xrr@{}}
\toprule
Support condition&Fixed fusion&Adaptive fusion\\
\midrule
Covered: $h^*\in\mathcal H$&72.7&76.2\\
Mild misspecification&69.2&70.4\\
Severe misspecification&67.6&67.2\\
\bottomrule
\end{tabularx}
\vspace{-0.1in}
\caption{Reported Flight final accuracy (\%), Gemma, under support misspecification. Both fusion variants retain the original EMA configuration; fixed fusion uses equal source weights.}
\label{tab:misspecification}
\vspace{-0.22in}
\end{table}

The adaptive advantage drops from 3.5 points with covered support to 1.2 under mild misspecification, then reverses by 0.4 under severe misspecification (\Cref{tab:misspecification}). Thus low entropy can favor an inaccurate symbolic explanation. A nonzero LLM weight does not prevent this failure. Densifying the support with interpolated vectors improves the reported mild result from 70.4\% to 73.0\%, suggesting that coverage can matter; it does not prove that support error is the only cause of the loss. Exact removal thresholds, seed-level uncertainty, and a no-EMA version of this probe remain necessary for a stronger robustness claim.

Support size also affects finite-round inference. On Flight, $M=24,48,96,192$ yields 71.8\%, 74.6\%, 76.2\%, and 76.0\%, respectively. Adding hypotheses may improve coverage but can also spread posterior mass across similar explanations. Size alone therefore does not determine quality; the placement of hypotheses matters. Under misspecification, the most plausible hypothesis is defined by expected choice likelihood, not generally by Euclidean distance to $h^*$ (\Cref{app:discrete_approx}).

For $\beta=4,6,8$, the reported final accuracies are 75.6\%, 76.2\%, and 75.4\%. These nearby point estimates indicate moderate local sensitivity around the chosen inverse temperature; extreme values 2 and 10 are reported to perform worse without a numeric breakdown. Since $\beta$ affects both posterior discrimination and predictive entropy, it also changes fusion weights. Stable accuracy in this range does not establish calibration of either source, and no held-out calibration metric is reported.

\section{Conclusion}
\label{sec:conclusion}
\vspace{-0.08in}

We presented \textsc{AdaptFuse}, a training-free framework for sequential
preference inference that combines symbolic Bayesian belief tracking with LLM semantic reasoning via entropy-adaptive fusion.
Across three domains and three base models, \textsc{AdaptFuse}
consistently outperforms both prompting baselines and fine-tuned methods
without any weight updates, demonstrating that externalizing
probabilistic computation is a viable and effective alternative to
fine-tuning for personalized recommendation.
Operating entirely at inference time on frozen lightweight open-source
models that can be served on-premise, \textsc{AdaptFuse} requires no
transmission, collection, or storage of sensitive user interaction data,
offering a practical path toward fully privacy-preserving personalization.

\clearpage
\section*{Limitations}
Our primary evaluation assumes fixed linear preferences and a training-derived support that covers test types. It does not establish generalization to arbitrary continuous or nonlinear preferences, preference drift, or systematic parser errors. The WebShopping feature schema and goal-to-hypothesis mapping are under-specified relative to Flight and Hotel, and preferred-option feedback must be distinguished from correctness-only feedback. No symbolic-only result, true-preference ceiling, or compute-matched comparison is available. Main tables lack seed-level uncertainty, so small margins should not be interpreted as statistically established. Legacy EMA results rely on option positions across changing candidate sets; their gains are not evidence for candidate-aligned smoothing. The no-EMA core is supported by Gemma ablations, with no corresponding Llama or Qwen measurements. The support-removal probe lacks a fully specified distance threshold and does not cover broader forms of model misspecification. Finally, configured domain parsers, likelihoods, and hypothesis annotations limit automatic transfer, and repeated LLM calls increase inference cost.


\bibliography{custom}

@article{brown2020language,
  title={Language models are few-shot learners},
  author={Brown, Tom and Mann, Benjamin and Ryder, Nick and Subbiah, Melanie and Kaplan, Jared D and Dhariwal, Prafulla and Neelakantan, Arvind and Shyam, Pranav and Sastry, Girish and Askell, Amanda and others},
  journal={Advances in neural information processing systems},
  volume={33},
  pages={1877--1901},
  year={2020}
}

@article{min2022rethinking,
  title={Rethinking the role of demonstrations: What makes in-context learning work?},
  author={Min, Sewon and Lyu, Xinxi and Holtzman, Ari and Artetxe, Mikel and Lewis, Mike and Hajishirzi, Hannaneh and Zettlemoyer, Luke},
  journal={arXiv preprint arXiv:2202.12837},
  year={2022}
}

@article{xie2021explanation,
  title={An explanation of in-context learning as implicit bayesian inference},
  author={Xie, Sang Michael and Raghunathan, Aditi and Liang, Percy and Ma, Tengyu},
  journal={arXiv preprint arXiv:2111.02080},
  year={2021}
}

@article{dai2022can,
  title={Why can gpt learn in-context? language models implicitly perform gradient descent as meta-optimizers},
  author={Dai, Damai and Sun, Yutao and Dong, Li and Hao, Yaru and Ma, Shuming and Sui, Zhifang and Wei, Furu},
  journal={arXiv preprint arXiv:2212.10559},
  year={2022}
}

@article{kojima2022large,
  title={Large language models are zero-shot reasoners},
  author={Kojima, Takeshi and Gu, Shixiang Shane and Reid, Machel and Matsuo, Yutaka and Iwasawa, Yusuke},
  journal={Advances in neural information processing systems},
  volume={35},
  pages={22199--22213},
  year={2022}
}

@article{wang2022self,
  title={Self-consistency improves chain of thought reasoning in language models},
  author={Wang, Xuezhi and Wei, Jason and Schuurmans, Dale and Le, Quoc and Chi, Ed and Narang, Sharan and Chowdhery, Aakanksha and Zhou, Denny},
  journal={arXiv preprint arXiv:2203.11171},
  year={2022}
}

@article{yao2023tree,
  title={Tree of thoughts: Deliberate problem solving with large language models},
  author={Yao, Shunyu and Yu, Dian and Zhao, Jeffrey and Shafran, Izhak and Griffiths, Tom and Cao, Yuan and Narasimhan, Karthik},
  journal={Advances in neural information processing systems},
  volume={36},
  pages={11809--11822},
  year={2023}
}

@article{wei2022chain,
  title={Chain-of-thought prompting elicits reasoning in large language models},
  author={Wei, Jason and Wang, Xuezhi and Schuurmans, Dale and Bosma, Maarten and Xia, Fei and Chi, Ed and Le, Quoc V and Zhou, Denny and others},
  journal={Advances in neural information processing systems},
  volume={35},
  pages={24824--24837},
  year={2022}
}

@article{nafar2023teaching,
  title={Teaching probabilistic logical reasoning to transformers},
  author={Nafar, Aliakbar and Venable, Kristen Brent and Kordjamshidi, Parisa},
  journal={arXiv preprint arXiv:2305.13179},
  year={2023}
}

@article{kadavath2022language,
  title={Language models (mostly) know what they know},
  author={Kadavath, Saurav and Conerly, Tom and Askell, Amanda and Henighan, Tom and Drain, Dawn and Perez, Ethan and Schiefer, Nicholas and Hatfield-Dodds, Zac and DasSarma, Nova and Tran-Johnson, Eli and others},
  journal={arXiv preprint arXiv:2207.05221},
  year={2022}
}

@article{pournemat2025reasoning,
  title={Reasoning Under Uncertainty: Exploring Probabilistic Reasoning Capabilities of LLMs},
  author={Pournemat, Mobina and Rezaei, Keivan and Sriramanan, Gaurang and Zarei, Arman and Fu, Jiaxiang and Wang, Yang and Eghbalzadeh, Hamid and Feizi, Soheil},
  journal={arXiv preprint arXiv:2509.10739},
  year={2025}
}

@inproceedings{gupta2025enough,
  title={Enough coin flips can make llms act bayesian},
  author={Gupta, Ritwik and Corona, Rodolfo and Ge, Jiaxin and Wang, Eric and Klein, Dan and Darrell, Trevor and Chan, David M},
  booktitle={Proceedings of the 63rd Annual Meeting of the Association for Computational Linguistics (Volume 1: Long Papers)},
  pages={7634--7655},
  year={2025}
}

@article{qiu2026bayesian,
  title={Bayesian teaching enables probabilistic reasoning in large language models},
  author={Qiu, Linlu and Sha, Fei and Allen, Kelsey and Kim, Yoon and Linzen, Tal and van Steenkiste, Sjoerd},
  journal={Nature Communications},
  year={2026},
  publisher={Nature Publishing Group UK London}
}

@article{feng2024bird,
  title={Bird: A trustworthy bayesian inference framework for large language models},
  author={Feng, Yu and Zhou, Ben and Lin, Weidong and Roth, Dan},
  journal={arXiv preprint arXiv:2404.12494},
  year={2024}
}

@article{zheng2025probabilistic,
  title={Probabilistic Reasoning with LLMs for k-anonymity Estimation},
  author={Zheng, Jonathan and Das, Sauvik and Ritter, Alan and Xu, Wei},
  journal={arXiv preprint arXiv:2503.09674},
  year={2025}
}

@article{rahimirad2025bayesian,
  title={Bayesian Social Deduction with Graph-Informed Language Models},
  author={Rahimirad, Shahab and Gergerli, Guven and Romero, Lucia and Qian, Angela and Olson, Matthew Lyle and Stepputtis, Simon and Campbell, Joseph},
  journal={arXiv preprint arXiv:2506.17788},
  year={2025}
}

@article{forouzandeh2025learning,
  title={Learning Hierarchical Procedural Memory for LLM Agents through Bayesian Selection and Contrastive Refinement},
  author={Forouzandeh, Saman and Peng, Wei and Moradi, Parham and Yu, Xinghuo and Jalili, Mahdi},
  journal={arXiv preprint arXiv:2512.18950},
  year={2025}
}

@article{domke2025large,
  title={Large language bayes},
  author={Domke, Justin},
  journal={arXiv preprint arXiv:2504.14025},
  year={2025}
}

@article{carpenter2017stan,
  title={Stan: A probabilistic programming language},
  author={Carpenter, Bob and Gelman, Andrew and Hoffman, Matthew D and Lee, Daniel and Goodrich, Ben and Betancourt, Michael and Brubaker, Marcus and Guo, Jiqiang and Li, Peter and Riddell, Allen},
  journal={Journal of statistical software},
  volume={76},
  pages={1--32},
  year={2017}
}

@article{kanda2025refinestat,
  title={REFINESTAT: Efficient Exploration for Probabilistic Program Synthesis},
  author={Kanda, Madhav and Ugare, Shubham and Misailovic, Sasa},
  journal={arXiv preprint arXiv:2509.01082},
  year={2025}
}

@inproceedings{lin2022inferring,
  title={Inferring rewards from language in context},
  author={Lin, Jessy and Fried, Daniel and Klein, Dan and Dragan, Anca},
  booktitle={Proceedings of the 60th Annual Meeting of the Association for Computational Linguistics (Volume 1: Long Papers)},
  pages={8546--8560},
  year={2022}
}

@article{yao2022webshop,
  title={Webshop: Towards scalable real-world web interaction with grounded language agents},
  author={Yao, Shunyu and Chen, Howard and Yang, John and Narasimhan, Karthik},
  journal={Advances in Neural Information Processing Systems},
  volume={35},
  pages={20744--20757},
  year={2022}
}

@article{team2024gemma,
  title={Gemma 2: Improving open language models at a practical size, 2024},
  author={Team, Gemma and Riviere, Morgane and Pathak, Shreya and Sessa, Pier Giuseppe and Hardin, Cassidy and Bhupatiraju, Surya and Hussenot, L{\'e}onard and Mesnard, Thomas and Shahriari, Bobak and Ram{\'e}, Alexandre and others},
  journal={URL https://arxiv. org/abs/2408.00118},
  volume={1},
  number={3},
  year={2024}
}

@article{grattafiori2024llama,
  title={The llama 3 herd of models},
  author={Grattafiori, Aaron and Dubey, Abhimanyu and Jauhri, Abhinav and Pandey, Abhinav and Kadian, Abhishek and Al-Dahle, Ahmad and Letman, Aiesha and Mathur, Akhil and Schelten, Alan and Vaughan, Alex and others},
  journal={arXiv preprint arXiv:2407.21783},
  year={2024}
}

@article{hui2024qwen2,
  title={Qwen2. 5-coder technical report},
  author={Hui, Binyuan and Yang, Jian and Cui, Zeyu and Yang, Jiaxi and Liu, Dayiheng and Zhang, Lei and Liu, Tianyu and Zhang, Jiajun and Yu, Bowen and Lu, Keming and others},
  journal={arXiv preprint arXiv:2409.12186},
  year={2024}
}

@article{murtaza2022ai,
  title={AI-based personalized e-learning systems: Issues, challenges, and solutions},
  author={Murtaza, Mir and Ahmed, Yamna and Shamsi, Jawwad Ahmed and Sherwani, Fahad and Usman, Mariam},
  journal={IEEE access},
  volume={10},
  pages={81323--81342},
  year={2022},
  publisher={IEEE}
}

@article{lindgren2025emerging,
  title={Emerging roles and relationships among humans and interactive AI systems},
  author={Lindgren, Helena},
  journal={International Journal of Human--Computer Interaction},
  volume={41},
  number={17},
  pages={10595--10617},
  year={2025},
  publisher={Taylor \& Francis}
}

@article{huang2025recommender,
  title={Recommender ai agent: Integrating large language models for interactive recommendations},
  author={Huang, Xu and Lian, Jianxun and Lei, Yuxuan and Yao, Jing and Lian, Defu and Xie, Xing},
  journal={ACM Transactions on Information Systems},
  volume={43},
  number={4},
  pages={1--33},
  year={2025},
  publisher={ACM New York, NY}
}

@inproceedings{cai2025agentic,
  title={Agentic feedback loop modeling improves recommendation and user simulation},
  author={Cai, Shihao and Zhang, Jizhi and Bao, Keqin and Gao, Chongming and Wang, Qifan and Feng, Fuli and He, Xiangnan},
  booktitle={Proceedings of the 48th International ACM SIGIR conference on Research and Development in Information Retrieval},
  pages={2235--2244},
  year={2025}
}

@inproceedings{huang2026towards,
  title={Towards next-generation recommender systems: A benchmark for personalized recommendation assistant with llms},
  author={Huang, Jiani and Wang, Shijie and Ning, Liangbo and Fan, Wenqi and Wang, Shuaiqiang and Yin, Dawei and Li, Qing},
  booktitle={Proceedings of the Nineteenth ACM International Conference on Web Search and Data Mining},
  pages={217--226},
  year={2026}
}

@article{gao2024aligning,
  title={Aligning llm agents by learning latent preference from user edits},
  author={Gao, Ge and Taymanov, Alexey and Salinas, Eduardo and Mineiro, Paul and Misra, Dipendra},
  journal={Advances in neural information processing systems},
  volume={37},
  pages={136873--136896},
  year={2024}
}

@inproceedings{milani2025aligning,
  title={Aligning agent policies with preferences: Human-centered interpretable reinforcement learning},
  author={Milani, Stephanie and Zhang, Zhicheng and Topin, Nicholay and Xia, Lirong and Fang, Fei},
  booktitle={Proceedings of the AAAI/ACM Conference on AI, Ethics, and Society},
  volume={8},
  number={2},
  pages={1711--1723},
  year={2025}
}

@article{he2025emerged,
  title={The emerged security and privacy of llm agent: A survey with case studies},
  author={He, Feng and Zhu, Tianqing and Ye, Dayong and Liu, Bo and Zhou, Wanlei and Yu, Philip S},
  journal={ACM Computing Surveys},
  volume={58},
  number={6},
  pages={1--36},
  year={2025},
  publisher={ACM New York, NY}
}

@article{yao2024survey,
  title={A survey on large language model (llm) security and privacy: The good, the bad, and the ugly},
  author={Yao, Yifan and Duan, Jinhao and Xu, Kaidi and Cai, Yuanfang and Sun, Zhibo and Zhang, Yue},
  journal={High-Confidence Computing},
  volume={4},
  number={2},
  pages={100211},
  year={2024},
  publisher={Elsevier}
}

@article{gan2024navigating,
  title={Navigating the risks: A survey of security, privacy, and ethics threats in llm-based agents},
  author={Gan, Yuyou and Yang, Yong and Ma, Zhe and He, Ping and Zeng, Rui and Wang, Yiming and Li, Qingming and Zhou, Chunyi and Li, Songze and Wang, Ting and others},
  journal={arXiv preprint arXiv:2411.09523},
  year={2024}
}

@article{mimno2012topic,
  title={Topic models conditioned on arbitrary features with dirichlet-multinomial regression},
  author={Mimno, David and McCallum, Andrew},
  journal={arXiv preprint arXiv:1206.3278},
  year={2012}
}

@article{holmes2012dirichlet,
  title={Dirichlet multinomial mixtures: generative models for microbial metagenomics},
  author={Holmes, Ian and Harris, Keith and Quince, Christopher},
  journal={PloS one},
  volume={7},
  number={2},
  pages={e30126},
  year={2012},
  publisher={Public Library of Science San Francisco, USA}
}

@article{zaiser2023exact,
  title={Exact Bayesian inference on discrete models via probability generating functions: a probabilistic programming approach},
  author={Zaiser, Fabian and Murawski, Andrzej and Ong, Chih-Hao Luke},
  journal={Advances in Neural Information Processing Systems},
  volume={36},
  pages={2427--2462},
  year={2023}
}

@inproceedings{fayyad1993multi,
  title={Multi-interval discretization of continuous-valued attributes for classification learning},
  author={Fayyad, Usama M and Irani, Keki B and others},
  booktitle={Ijcai},
  volume={93},
  number={2},
  pages={1022--1029},
  year={1993}
}

@article{luce1977choice,
  title={The choice axiom after twenty years},
  author={Luce, R Duncan},
  journal={Journal of mathematical psychology},
  volume={15},
  number={3},
  pages={215--233},
  year={1977},
  publisher={Elsevier}
}

@book{luce1959individual,
  title={Individual choice behavior},
  author={Luce, R Duncan and others},
  volume={4},
  year={1959},
  publisher={Wiley New York}
}

@inproceedings{madsen2005modeling,
  title={Modeling word burstiness using the Dirichlet distribution},
  author={Madsen, Rasmus E and Kauchak, David and Elkan, Charles},
  booktitle={Proceedings of the 22nd international conference on Machine learning},
  pages={545--552},
  year={2005}
}

@inproceedings{elkan2006clustering,
  title={Clustering documents with an exponential-family approximation of the Dirichlet compound multinomial distribution},
  author={Elkan, Charles},
  booktitle={Proceedings of the 23rd international conference on Machine learning},
  pages={289--296},
  year={2006}
}

@misc{minka2000estimating,
  title={Estimating a Dirichlet distribution},
  author={Minka, Thomas},
  year={2000},
  publisher={Technical report, MIT}
}

@article{doob1949heuristic,
  title={Heuristic approach to the Kolmogorov-Smirnov theorems},
  author={Doob, Joseph L},
  journal={The Annals of Mathematical Statistics},
  pages={393--403},
  year={1949},
  publisher={JSTOR}
}

@inproceedings{ghosal1997review,
  title={A review of consistency and convergence of posterior distribution},
  author={Ghosal, Subhashis},
  booktitle={Varanashi Symposium in Bayesian Inference, Banaras Hindu University},
  year={1997}
}

@book{lindley1972bayesian,
  title={Bayesian statistics: A review},
  author={Lindley, Dennis Victor},
  year={1972},
  publisher={SIAM}
}

@article{hansen2001acknowledging,
  title={Acknowledging misspecification in macroeconomic theory},
  author={Hansen, Lars Peter and Sargent, Thomas J},
  journal={Review of Economic Dynamics},
  volume={4},
  number={3},
  pages={519--535},
  year={2001},
  publisher={Elsevier}
}

@article{kleijn2012bernstein,
  title={The Bernstein-von-Mises theorem under misspecification},
  author={Kleijn, Bas JK and Van der Vaart, Aad W},
  year={2012}
}

@inproceedings{perez2008kullback,
  title={Kullback-Leibler divergence estimation of continuous distributions},
  author={P{\'e}rez-Cruz, Fernando},
  booktitle={2008 IEEE international symposium on information theory},
  pages={1666--1670},
  year={2008},
  organization={IEEE}
}

@article{zhang2023properties,
  title={On the properties of kullback-leibler divergence between multivariate gaussian distributions},
  author={Zhang, Yufeng and Pan, Jialu and Li, Li Ken and Liu, Wanwei and Chen, Zhenbang and Liu, Xinwang and Wang, Ji},
  journal={Advances in neural information processing systems},
  volume={36},
  pages={58152--58165},
  year={2023}
}

@article{achiam2023gpt,
  title={Gpt-4 technical report},
  author={Achiam, Josh and Adler, Steven and Agarwal, Sandhini and Ahmad, Lama and Akkaya, Ilge and Aleman, Florencia Leoni and Almeida, Diogo and Altenschmidt, Janko and Altman, Sam and Anadkat, Shyamal and others},
  journal={arXiv preprint arXiv:2303.08774},
  year={2023}
}

@article{hurst2024gpt,
  title={Gpt-4o system card},
  author={Hurst, Aaron and Lerer, Adam and Goucher, Adam P and Perelman, Adam and Ramesh, Aditya and Clark, Aidan and Ostrow, AJ and Welihinda, Akila and Hayes, Alan and Radford, Alec and others},
  journal={arXiv preprint arXiv:2410.21276},
  year={2024}
}

@article{cheng2024exploring,
  title={Exploring large language model based intelligent agents: Definitions, methods, and prospects},
  author={Cheng, Yuheng and Zhang, Ceyao and Zhang, Zhengwen and Meng, Xiangrui and Hong, Sirui and Li, Wenhao and Wang, Zihao and Wang, Zekai and Yin, Feng and Zhao, Junhua and others},
  journal={arXiv preprint arXiv:2401.03428},
  year={2024}
}

@article{jiang2025timepre,
  title={TimePre: Bridging Accuracy, Efficiency, and Stability in Probabilistic Time-Series Forecasting},
  author={Jiang, Lingyu and Xu, Lingyu and Li, Peiran and Ge, Qianwen and Zhuang, Dingyi and Xing, Shuo and Chen, Wenjing and Gao, Xiangbo and Chen, Ting-Hsuan and Zhan, Xueying and others},
  journal={arXiv preprint arXiv:2511.18539},
  year={2025}
}
\clearpage
\appendix

\section{Complete Algorithm}
\label{app:algorithm}
Algorithm~\ref{alg:adaptfuse} summarizes the complete inference
procedure of \textsc{AdaptFuse}. Given a fixed hypothesis set and
a frozen LLM, the method maintains a Bayesian belief over user
preferences across interaction rounds. At each round, it combines
the symbolic predictive distribution with aggregated LLM responses
through entropy-adaptive fusion, then updates the belief after
observing the user's choice. After the final interaction, the
posterior and aggregation state are held fixed for prediction on
unseen candidate sets.
\begin{algorithm}[H]
\small
\caption{\textsc{AdaptFuse}}
\label{alg:adaptfuse}
\begin{algorithmic}[1]
\Require $\mathcal{H}$,
         rounds $\{(X_t, y_t)\}_{t=1}^T$,
         held-out sets $\{\tilde{X}_s\}$,
         frozen LLM
\Ensure $\{i^*_t\}_{t=1}^T$, $\{\tilde{i}^*_s\}$

\State $\mathbf{b} \leftarrow \mathbf{1}_M/M$;\enspace
       $\mathrm{mem} \leftarrow \texttt{null}$

\For{$t = 1, \ldots, T$}
    \State \textit{// Stage 1: predict using pre-update belief}
    \State $\pi^{\mathrm{sym}} \leftarrow
        \bigl[\sum_m b_m \cdot P(i \mid X_t, h_m)\bigr]_{i=1}^K$
        \hfill\textit{// Eq.~\eqref{eq:predictive}}

    \State \textit{// Stage 2}
    \State $\{(\widehat{y}_s, c_s)\}_{s=1}^N
        \leftarrow \textsc{SampleLLM}(X_t,\,\mathrm{history})$
    \State $\boldsymbol{\alpha} \leftarrow \alpha_0\mathbf{1}_K$;\enspace
        $\forall$ valid $s$: update $\alpha_i$ via
        Eq.~\eqref{eq:dirichlet_update}
    \State $\pi^{\mathrm{llm}}_{\mathrm{raw}} \leftarrow
        \boldsymbol{\alpha}/\|\boldsymbol{\alpha}\|_1$
    \If{$\mathrm{mem} \neq \texttt{null}$}
        \State $\pi^{\mathrm{llm}} \leftarrow
            \mathrm{normalize}(m\cdot\mathrm{mem}
            + (1{-}m)\cdot\pi^{\mathrm{llm}}_{\mathrm{raw}})$
            \hfill\textit{// Eq.~\eqref{eq:ema}}
    \Else\enspace $\pi^{\mathrm{llm}} \leftarrow
        \pi^{\mathrm{llm}}_{\mathrm{raw}}$
    \EndIf
    \State $\mathrm{mem} \leftarrow \pi^{\mathrm{llm}}$
        \Comment{write only on interaction rounds}

    \State \textit{// Stage 3}
    \State $w^{\mathrm{llm}},\,w^{\mathrm{sym}}$
        via Eq.~\eqref{eq:adaptive_weights}
    \State $\pi^* \leftarrow \mathrm{normalize}(
        w^{\mathrm{llm}}\pi^{\mathrm{llm}}
        + w^{\mathrm{sym}}\pi^{\mathrm{sym}})$
    \State $i^*_t \leftarrow \operatorname{argmax}_i\,\pi^*_i$

    \State \textit{// Update belief after observing $y_t$}
    \State $b_m \leftarrow b_m\cdot\max(\varepsilon, P(y_t|X_t,h_m))$
        $\,\forall m$;\enspace
        $\mathbf{b} \leftarrow \mathbf{b}/\|\mathbf{b}\|_1$
        \hfill\textit{// Eq.~\eqref{eq:bayes}}
\EndFor

\For{each held-out $\tilde{X}_s$}
    \State Compute $\pi^*$ with frozen $\mathbf{b}$,
        read-only $\mathrm{mem}$\enspace (\textit{no update})
    \State $\tilde{i}^*_s \leftarrow \operatorname{argmax}_i\,\pi^*_i$
\EndFor
\end{algorithmic}

\end{algorithm}

\section{Proofs}
\label{app:proofs}

\subsection{Proof of Lemma~\ref{lemma:dirichlet}}
\label{app:proof_dirichlet}

\begin{lemma}[Dirichlet Aggregation as Posterior Mean]
\label{lemma:dirichlet}
Let $\theta \sim \mathrm{Dir}(\alpha_0 \mathbf{1}_K)$ be a prior over the
categorical parameter. Treating the confidence-weighted vectors
$\{w_{s,i}\}$ as fractional pseudo-count observations—a standard
extension of Dirichlet-Multinomial conjugacy~\citep{elkan2006clustering,minka2000estimating}—
the posterior is $\theta \mid \{w_{s,i}\} \sim \mathrm{Dir}(\boldsymbol{\alpha})$,
and
$\pi^{\mathrm{llm}}_{\mathrm{raw},t} = \mathbb{E}[\theta \mid \{w_{s,i}\}]$.
The proof follows below.
\end{lemma}

\begin{proof}
\textit{Step 1: validity of pseudo-counts.}
For each sample $s$ and option $i$, $w_{s,i} \geq 0$ since
$c_s \in [0,1]$.
The weights sum to one:
\[
    \sum_{i=1}^{K} w_{s,i}
    = c_s + (K-1) \cdot \frac{1-c_s}{K-1}
    = 1.
\]
Hence $\{w_{s,i}\}_{i=1}^{K}$ is a valid probability vector for each $s$.
Treating it as a fractional pseudo-count observation is a standard
extension of the Dirichlet-Multinomial model that preserves the conjugate
update structure~\citep{minka2000estimating,elkan2006clustering}.

\textit{Step 2: conjugate posterior.}
Let $\theta \in \Delta^{K-1}$ be the unknown categorical parameter with
prior $\theta \sim \mathrm{Dir}(\alpha_0 \mathbf{1}_K)$.
Accumulating pseudo-count vectors $\{w_{s,i}\}$ over valid samples
yields the posterior:
\[
    \theta \mid \{w_{s,i}\}
    \;\sim\; \mathrm{Dir}(\boldsymbol{\alpha}),
    \qquad
    \alpha_i = \alpha_0 + \textstyle\sum_{s} w_{s,i},
\]
by conjugacy of the Dirichlet family with respect to (fractional)
count observations.
The posterior mean is:
\[
    \mathbb{E}[\theta_i \mid \{w_{s,i}\}]
    = \frac{\alpha_i}{\sum_j \alpha_j}
    = \pi^{\mathrm{llm}}_{\mathrm{raw},i}.
    \qedhere
\]
\end{proof}

\subsection{Adaptive Fusion Schedule}
\label{app:proof_schedule}

\begin{proposition}[Fusion Bound]
\label{prop:schedule}
At any round $t$, the LLM's share of the fused prediction satisfies:
\begin{equation}
    \frac{w^{\mathrm{llm}}_t}{w^{\mathrm{llm}}_t + w^{\mathrm{sym}}_t}
    \;\leq\; \frac{1}{1 + w^{\mathrm{sym}}_t}.
    \label{eq:fusion_bound}
\end{equation}
This bound is monotonically decreasing in $w^{\mathrm{sym}}_t$.
In particular:
\begin{enumerate}
    \item When the symbolic posterior is diffuse
        ($\tilde{H}(\pi^{\mathrm{sym}}) \approx 1$, i.e.,
        $w^{\mathrm{sym}}_t \approx \varepsilon_w$), the bound
        approaches $1/(1+\varepsilon_w) \approx 1$, and the two
        sources contribute roughly equally.
    \item When the symbolic posterior is confident
        ($\tilde{H}(\pi^{\mathrm{sym}}) \approx 0$, i.e.,
        $w^{\mathrm{sym}}_t \approx 1$), the bound becomes
        $1/2$, guaranteeing that the symbolic signal receives at
        least half the total weight. If the LLM is simultaneously
        uncertain ($w^{\mathrm{llm}}_t \approx \varepsilon_w$),
        the ratio further shrinks to
        $\varepsilon_w/(\varepsilon_w + 1) \approx 0$.
\end{enumerate}
\end{proposition}

\begin{proof}
Since $w^{\mathrm{llm}}_t \in [\varepsilon_w, 1]$ by definition
(normalized entropy is in $[0,1]$, so
$1 - \tilde{H} \in [0,1]$, and the $\max$ with $\varepsilon_w$ ensures
positivity):
\[
    \frac{w^{\mathrm{llm}}_t}{w^{\mathrm{llm}}_t + w^{\mathrm{sym}}_t}
    \;\leq\; \frac{1}{1 + w^{\mathrm{sym}}_t},
\]
where the inequality uses $w^{\mathrm{llm}}_t \leq 1$.
The right-hand side is monotonically decreasing in $w^{\mathrm{sym}}_t$
since $w^{\mathrm{sym}}_t > 0$.

For item~1: when the symbolic posterior is nearly uniform,
$\tilde{H}(\pi^{\mathrm{sym}}) \approx 1$ and
$w^{\mathrm{sym}}_t = \max(\varepsilon_w, 1 - \tilde{H}) \approx \varepsilon_w$,
giving $1/(1+\varepsilon_w) \approx 1$.

For item~2: when the symbolic posterior concentrates on a single option,
$\tilde{H}(\pi^{\mathrm{sym}}) \approx 0$ and
$w^{\mathrm{sym}}_t \approx 1$, giving an upper bound of $1/(1+1)=1/2$.
For the tighter statement, substituting
$w^{\mathrm{llm}}_t = \varepsilon_w$ directly:
\[
    \frac{\varepsilon_w}{\varepsilon_w + 1} \;\approx\; 10^{-3}.
    \qedhere
\]
\end{proof}

\begin{remark}[Connection to posterior concentration]
\label{rmk:concentration}
As the number of interaction rounds grows, the Bayesian posterior
$\mathbf{b}_t$ concentrates on the true hypothesis $h^*$ (assuming
$h^* \in \mathcal{H}$; see \Cref{app:discrete_approx}).
This drives $\pi^{\mathrm{sym}}_t \to P(\cdot \mid X, h^*)$ and
$\tilde{H}(\pi^{\mathrm{sym}}_t) \to H^*_X \triangleq
\tilde{H}(P(\cdot \mid X, h^*))$.
For finite $\beta$, $H^*_X$ is a positive constant determined by
$\beta$ and the utility gaps among options; the limiting symbolic weight
is $w^{\mathrm{sym}} \to \max(\varepsilon_w,\, 1 - H^*_X)$.
With the default $\beta = 6.0$, we empirically observe that $H^*_X$ is
typically small across our evaluation domains, so $w^{\mathrm{sym}}$
approaches a value close to $1$ within the $T = 5$ interaction rounds
used in our experiments (see \Cref{sec:experiments}).
\end{remark}

\subsection{Validity of the Discrete Hypothesis Approximation}
\label{app:discrete_approx}

The use of a finite hypothesis set $\mathcal{H}$ raises the question of
whether Bayesian inference over $\mathcal{H}$ faithfully approximates
inference over the continuous space $\mathbb{R}^d$.
We discuss the two relevant cases below; the arguments are intended as
intuitive justifications rather than formal theorems.

\paragraph{Well-specified case: $h^* \in \mathcal{H}$.}
If the user's true preference $h^*$ is an element of $\mathcal{H}$,
then by standard Bayesian posterior consistency
arguments~\citep{doob1949heuristic,ghosal1997review}: the posterior concentrates on the
true data-generating parameter almost surely as $t \to \infty$, provided
that the true parameter lies in the support of the prior; we have
$b_{t, m^*} \to 1$ where $m^*$ satisfies $h_{m^*} = h^*$.
In our experimental setup, $\mathcal{H}$ is constructed from the
ground-truth preference vectors of all users in the training split.
The train/test split in each of our three datasets is organized so that
every evaluation user's preference type appears in the training set;
therefore $h^* \in \mathcal{H}$ holds by construction, and the uniform
prior $b_{0,m} = 1/M$ assigns positive mass to every hypothesis,
satisfying the support condition.

\paragraph{Misspecified case: $h^* \notin \mathcal{H}$.}
If $h^*$ does not belong to $\mathcal{H}$ (e.g., at test time in a new
domain), the posterior does not concentrate on the true parameter but
instead on the hypothesis that best explains the observed choices.
Under standard misspecification
theory~\citep{kleijn2012bernstein,lindley1972bayesian,hansen2001acknowledging}, the posterior
concentrates on the hypothesis
\[
    h_{m^*} = \operatorname{argmin}_{h_m \in \mathcal{H}} \;
    \mathrm{KL}\!\bigl(P_{\mathrm{true}}(\cdot)
    \;\|\; P(\cdot \mid \cdot, h_m)\bigr),
\]
i.e., the $h_m$ whose induced choice distribution is closest to the
true choice distribution in the Kullback--Leibler sense~\cite{perez2008kullback,zhang2023properties}.
Note that this is \emph{not} generally the Euclidean nearest neighbor
$\operatorname{argmin}_m \| h_m - h^* \|_2$; the relevant metric is
determined by the observation model (here, the softmax likelihood).
Prediction quality degrades gracefully with the KL approximation error,
suggesting that denser hypothesis sets (larger $M$) yield better
approximations when the test distribution shifts, at the cost of a
modest increase in computation.


\section{Task Details}
\label{app:task_details}

\subsection{Flight Recommendation}
\label{app:flight}

The flight recommendation task is derived from \citet{lin2022inferring}.
The agent interacts with a user over $T = 5$ rounds.
At each round, three flight options are presented to both the user and the
agent; each flight is characterized by four attributes: departure time,
duration, number of stops, and price.
Each user is associated with a fixed latent preference vector over these
attributes.
For each attribute, a user may exhibit a strong or weak preference for
either high or low values (e.g., preferring shorter flights or lower
prices), or no preference on that attribute.

After each interaction round, the agent's held-out prediction accuracy is
evaluated on 50 new sets of three flights drawn from the same distribution
but not shown during interaction.
This held-out evaluation is our primary metric: it measures how well the
agent has generalized the user's latent preference to unseen items, rather
than merely adapting to the specific options encountered during interaction.

Each flight option is presented to the agent as a short natural language
string deterministically generated from its attribute values, following the
template:
\begin{center}
\small
\texttt{Flight i: Departure time: HH:MM AM/PM, Duration: Xhr Ymin,}\\
\texttt{Number of stops: Z, Price: \$W}
\end{center}
Feature normalization maps departure time from $[06{:}00, 22{:}00]$,
duration from $[30\,\text{min}, 20\,\text{hr}]$, stops from $[0, 2]$,
and price from $[\$100, \$1000]$, each to $[0, 1]$.

\subsection{Hotel Recommendation}
\label{app:hotel}

The hotel recommendation task follows the same sequential interaction
protocol as the flight task~\citep{qiu2026bayesian}.
Each hotel is described by four attributes: distance to downtown (in km),
price per night (in USD), star rating (1--5), and number of amenities
(integer count).
As in the flight task, each hotel is presented as a short natural language
string deterministically generated from its attribute values:
\begin{center}
\small
\texttt{Hotel i: Distance to downtown: X km, Price: \$Y/night,}\\
\texttt{Rating: Z stars, Amenities: W}
\end{center}
Feature normalization maps distance from $[0.5, 20]$\,km, price from
$[\$50, \$500]$, rating from $[1, 5]$, and amenities from $[0, 10]$,
each to $[0, 1]$.

User preferences over hotel attributes are sampled from the same
preference model as the flight task, allowing direct comparison of
performance across domains with different attribute semantics.
The hypothesis set $\mathcal{H}$ is constructed separately for each domain
from the corresponding training split.

\subsection{Web Shopping}
\label{app:webshopping}

The WebShopping task uses real-world product data from the simulated
e-commerce environment of \citet{yao2022webshop}.
Each user is defined by a randomly sampled goal specifying desired product
characteristics, such as ``a machine-washable cotton shirt in size XL
under \$30''.

At each interaction round, the agent is presented with a set of products
randomly sampled from a fixed category (e.g., shirts) and must recommend
the most suitable product.
Each product is represented by a short title concatenated with a detailed
free-text description.
An example product representation is shown in \Cref{tab:webshopping_example}.

After each round, the user provides binary feedback (correct/incorrect).
The preferred product is defined as the one with the highest reward score
following the formulation of \citet{yao2022webshop}, which aggregates
attribute match, price, and rating into a scalar reward.

This task differs from the flight and hotel tasks in two key respects.
First, item descriptions are open-ended natural language rather than
structured attribute strings, making deterministic feature extraction
approximate.
Second, user goals are expressed as free-form text queries rather than
numerical preference vectors, introducing additional ambiguity in
preference modeling.
These factors make WebShopping substantially harder, as reflected
in the lower accuracy of all evaluated configurations on this task
(\Cref{tab:cross_domain_ablation}).

\begin{table}[t]
\centering\small
\setlength{\tabcolsep}{3pt}
\begin{tabularx}{\columnwidth}{@{}l>{\raggedright\arraybackslash}X@{}}
\toprule
\textbf{Field} & \textbf{Content} \\
\midrule
Title       & Men's Classic Fit Short-Sleeve Crewneck T-Shirt \\
Description & 100\% cotton. Machine wash cold. Available in sizes S, M, L, XL, XXL.
              Relaxed fit. Imported. \$24.99. \\
User goal   & Looking for a cotton shirt, machine-washable, size XL, under \$30. \\
\bottomrule
\end{tabularx}
\caption{Example product representation in the WebShopping task.}
\label{tab:webshopping_example}
\end{table}

\section{Dataset Statistics}
\label{app:dataset_statistics}

\Cref{tab:dataset_statistics} summarizes the three task datasets and the
hypothesis-support sizes used by \textsc{AdaptFuse}.
Flight and Hotel each contain 624 simulated user records.
WebShopping contains 990 user-goal records across 99 product categories,
with ten records per category~\citep{qiu2026bayesian}.
Each record contains a five-round interaction trajectory with three
candidate items per round.

\begin{table}[t]
\centering\small
\setlength{\tabcolsep}{3pt}
\begin{tabularx}{\columnwidth}{@{}>{\raggedright\arraybackslash}Xrrr@{}}
\toprule
Statistic & Flight & Hotel & WebShopping\\
\midrule
User records / trajectories & 624 & 624 & 990\\
Rounds per trajectory & 5 & 5 & 5\\
Candidates per option set & 3 & 3 & 3\\
Hypothesis-support size ($M$) & 96 & 112 & 187\\
\bottomrule
\end{tabularx}
\caption{Dataset statistics and domain-specific hypothesis-support sizes.
$M$ counts preference hypotheses, rather than users or trajectories.}
\label{tab:dataset_statistics}
\end{table}

\paragraph{Evaluation protocol.}
For Flight and Hotel, accuracy is measured on 50 held-out option sets
per user after each interaction round, following \Cref{sec:tasks}.
Held-out sets are excluded from the interaction history and do not
trigger belief updates. The hypothesis support is constructed separately
from each domain's training split; the main evaluation split covers
preference types represented in that support.

\section{Dataset Sources and Licenses}
\label{app:dataset_licenses}

This section documents the public sources and license notices relevant to
Flight, Hotel, and WebShopping, distinguishing the released task data from
upstream repositories.

\paragraph{Flight and Hotel task data.}
\citet{qiu2026bayesian} describe their Flight and Hotel data as synthetically
generated and release the training and evaluation data in
\emph{LLM Bayesian Teaching Data} on Zenodo
(DOI: \href{https://doi.org/10.5281/zenodo.17677329}{10.5281/zenodo.17677329}).
The record's metadata specifies Creative Commons Attribution 4.0
International (CC BY 4.0).\footnote{Dataset record and machine-readable
license metadata: \url{https://zenodo.org/records/17677329} and
\url{https://zenodo.org/api/records/17677329}.}
This license permits reuse and adaptation with appropriate attribution,
a link to the license, and an indication of changes.\footnote{\url{https://creativecommons.org/licenses/by/4.0/}}

\paragraph{Original Flight task source.}
The Flight task also draws on \citet{lin2022inferring}, whose official
\texttt{rewards-from-language} repository contains the FlightPref data and
code. The repository carries the MIT License.\footnote{\url{https://github.com/jlin816/rewards-from-language/blob/main/LICENSE}}
This upstream repository notice is distinct from the CC BY 4.0 license
of the synthetic task-data release above. The MIT License requires
retaining its copyright and permission notices in copies or substantial
portions of the licensed material.

\paragraph{WebShopping and WebShop.}
WebShopping uses products and goals from WebShop~\citep{yao2022webshop}.
The official WebShop repository carries the MIT License, with the same
notice-retention requirement.\footnote{\url{https://github.com/princeton-nlp/WebShop/blob/master/LICENSE.md}}
\citet{qiu2026bayesian} also release their adapted web-shopping data in
the CC BY 4.0 Zenodo deposit identified above.
These notices describe the repository and the derived release,
respectively; the repository's software license alone does not establish
licensing rights for all third-party product content.

\section{Implementation Details}
\label{sec:impl}

\paragraph{Model sizes and computational resources.}
\label{app:compute}
We use Gemma~2 9B, Llama~3 8B, and Qwen~2.5 7B, with approximately
9, 8, and 7 billion parameters, respectively. Our experiments were run
on a single NVIDIA A100 GPU, with a total computational budget of
126 GPU hours. \textsc{AdaptFuse} keeps all model parameters frozen
and introduces no additional trainable parameters.

\paragraph{Closed-form posterior.}
Unrolling the Bayesian update over $t$ rounds under the uniform prior gives:
\begin{equation}
    b_{t,m}
    = \frac{\prod_{\tau=1}^{t} \ell_{\tau,m}}
           {\sum_{m'=1}^{M} \prod_{\tau=1}^{t} \ell_{\tau,m'}}.
    \label{eq:posterior_closed}
\end{equation}

\paragraph{Temporal exponential moving average smoothing.}
Per-round LLM aggregates exhibit high variance due to small $N$.
To reduce this variance, \textsc{AdaptFuse} applies exponential
smoothing across interaction rounds:
\begin{equation}
    \pi^{\mathrm{llm}}_t =
    \operatorname{normalize}\!\left(
        m\pi^{\mathrm{llm}}_{t-1}
        + (1-m)\pi^{\mathrm{llm}}_{\mathrm{raw},t}
    \right),
    \label{eq:ema}
\end{equation}
where $m = 0.65$ is the EMA coefficient.
For the first interaction round, we initialize
$\pi^{\mathrm{llm}}_t = \pi^{\mathrm{llm}}_{\mathrm{raw},t}$.
The smoothing yields an asymptotic effective sample count
$N_{\mathrm{eff}} = N/(1-m) \approx 14$
(the finite-horizon effective count is smaller in early rounds
but approaches this value within approximately 5 rounds).
The smoothed distribution is updated only on interaction rounds;
during held-out evaluation, it is read but not updated,
preventing held-out noise from contaminating the running average.

\Cref{tab:hyperparams} lists all hyperparameters of \textsc{AdaptFuse}.
$\beta = 6.0$ was set to match the empirical sharpness of user choices in
the dataset: lower values produce likelihoods too flat to discriminate
between hypotheses after one round; higher values risk premature posterior
collapse.
$m = 0.65$ yields an effective LLM sample count of
$N_{\mathrm{eff}} = N/(1-m) \approx 14$, substantially reducing per-round
sampling variance.
No hyperparameter was tuned on held-out data.

\begin{table}[]
\centering\footnotesize
\setlength{\tabcolsep}{2pt}
\begin{tabularx}{\columnwidth}{@{}l>{\raggedright\arraybackslash}Xc>{\raggedright\arraybackslash}Xl@{}}
\toprule
\textbf{Symbol} & \textbf{Parameter} & \textbf{Value}
    & \textbf{Role} & \textbf{Eq.} \\
\midrule
$\beta$          & Softmax inverse temp. & $6.0$
    & Likelihood sharpness     & \eqref{eq:likelihood} \\
$\alpha_0$       & Dirichlet prior       & $1.0$
    & Laplace smoothing        & \eqref{eq:dirichlet_update} \\
$m$              & EMA coefficient          & $0.65$
    & LLM variance reduction   & \eqref{eq:ema} \\
$N$              & LLM samples/round     & $5$
    & Aggregation sample count & \eqref{eq:dirichlet_update} \\
$\varepsilon_w$  & Fusion weight floor   & $10^{-3}$
    & Prevents zero weights    & \eqref{eq:adaptive_weights} \\
$\varepsilon$    & Likelihood floor      & $10^{-8}$
    & Numerical stability      & \eqref{eq:bayes} \\
\bottomrule
\end{tabularx}
\caption{\textsc{AdaptFuse} hyperparameters. No parameter was tuned on held-out data.}
\label{tab:hyperparams}
\end{table}

\section{Prompt Details}
\label{sec:prompt}

\begin{paperprompt}{Example: Two-round Flight Recommendation Interaction}
\small

\PromptRole{User}
Help me select the best flights for my trips. I have specific preferences for what I like and dislike in a flight, and these preferences remain the same. You need to figure out my preferences and select the best flights for me. Use your best judgment if you are unsure. Do not say you need more information.

\PromptField{Round 1: Initial Selection}
\textbf{Flight Options:} \\
Flight 1: departure time: 02:00 PM, duration: 2 hr 30 min, number of stops: 1, price: \$370 \\
Flight 2: departure time: 10:00 PM, duration: 4 hr 24 min, number of stops: 0, price: \$730 \\
Flight 3: departure time: 03:36 PM, duration: 16 hr 6 min, number of stops: 0, price: \$900 \\

\textbf{Model Response:} The best option is Flight 1.

\PromptField{Round 2: Feedback and Adaptation}
\textbf{User Feedback:} Your option Flight 1 is incorrect. I prefer Flight 2. \\
\\
\textbf{New Flight Options:} \\
Flight 1: departure time: 04:00 PM, duration: 18 hr 3 min, number of stops: 2, price: \$280 \\
Flight 2: departure time: 10:48 AM, duration: 6 hr 21 min, number of stops: 1, price: \$370 \\
Flight 3: departure time: 06:48 PM, duration: 10 hr 5 min, number of stops: 1, price: \$810 \\

\PromptField{Final Result}
\textbf{Model Response:} The best option is Flight 2. \\
\textbf{User Status:} Your option Flight 2 is correct.

\end{paperprompt}

\begin{paperprompt}{Example: Hotel Recommendation Task Interaction}
\small

\PromptRole{User}
Help me select the best hotels for my trips. I have specific preferences for what I like and dislike in a hotel, and these preferences remain the same. You need to figure out my preferences and select the best hotels for me. Use your best judgment if you are unsure. Do not say you need more information.

\PromptField{Round 1: Hotel Selection}
\textbf{Hotel Options:} \\
Hotel 1: distance to downtown: 4 miles, price: \$550, rating: 3 stars, amenities: free parking and free breakfast \\
Hotel 2: distance to downtown: 3 miles, price: \$820, rating: 2 stars, amenities: free parking, free breakfast, and pool \\
Hotel 3: distance to downtown: 2.3 miles, price: \$370, rating: 1 stars, amenities: free parking \\

\PromptField{Model Response}
The best option is Hotel 3.

\PromptField{User Feedback}
Your option Hotel 3 is incorrect. I prefer Hotel 2.

\end{paperprompt}

\begin{paperprompt}{Example: Web Shopping Task Interaction}
\small

\PromptRole{User}
Help me select the best product. I have specific preferences for what I like and dislike in a product, and these preferences remain the same. You need to figure out my preferences and select the best products for me. Use your best judgment if you are unsure. Do not say you need more information.

\PromptField{Product Options}
\textbf{Product 1:} \\
\textit{Title:} Chic D Independence Day Table Runner 72 Inches Long, Gnome Cotton Linen Spring Table Cloth Runners... \\
\textit{Description:} 14x72inch Dining Table Runner Size; High Quality Cotton Linen; Perfect for holidays, catering, BBQ's, etc. \\
\textit{Color:} black white \quad \textit{Size:} 13x108inch \\
\\
\textbf{Product 2:} \\
\textit{Title:} Ambesonne Orange Mandala Coffee Table, Pastel Colored Flourishes and Dark Toned Details... \\
\textit{Description:} 24" Long x 18" Wide x 15" High; High Quality Beech Wooden Frame and Acrylic Glass Table Top; Easy to assembly. \\
\textit{Color:} blue purple \quad \textit{Size:} large \\
\\
\textbf{Product 3:} \\
\textit{Title:} White Round Dining Table and 4 Chairs, Mid-Century Modern Coffee Table Round Kitchen Table... \\
\textit{Description:} Table size 35.4*35.4*29.5 inch; Clear glass top with solid wood metal legs; Chairs made of velvet. \\
\textit{Size:} round table with wood legs

\PromptField{Model Response}
The best option is Product 3.

\PromptField{User Feedback}
Your option Product 3 is incorrect. I prefer Product 2.

\end{paperprompt}

\section{More results}

\textbf{Flight recommendation.} \Cref{fig:flight_main} shows held-out accuracy on the flight task across all models and baselines. \textsc{AdaptFuse} achieves the highest accuracy on all three base models, outperforming direct prompting, CoT, and self-consistency by substantial margins, and surpassing fine-tuned Bayesian Teaching models despite requiring no weight updates. \textsc{AdaptFuse}'s LLM integration provides a consistent improvement.

\begin{figure*}[]
\centering
\includegraphics[width=\linewidth]{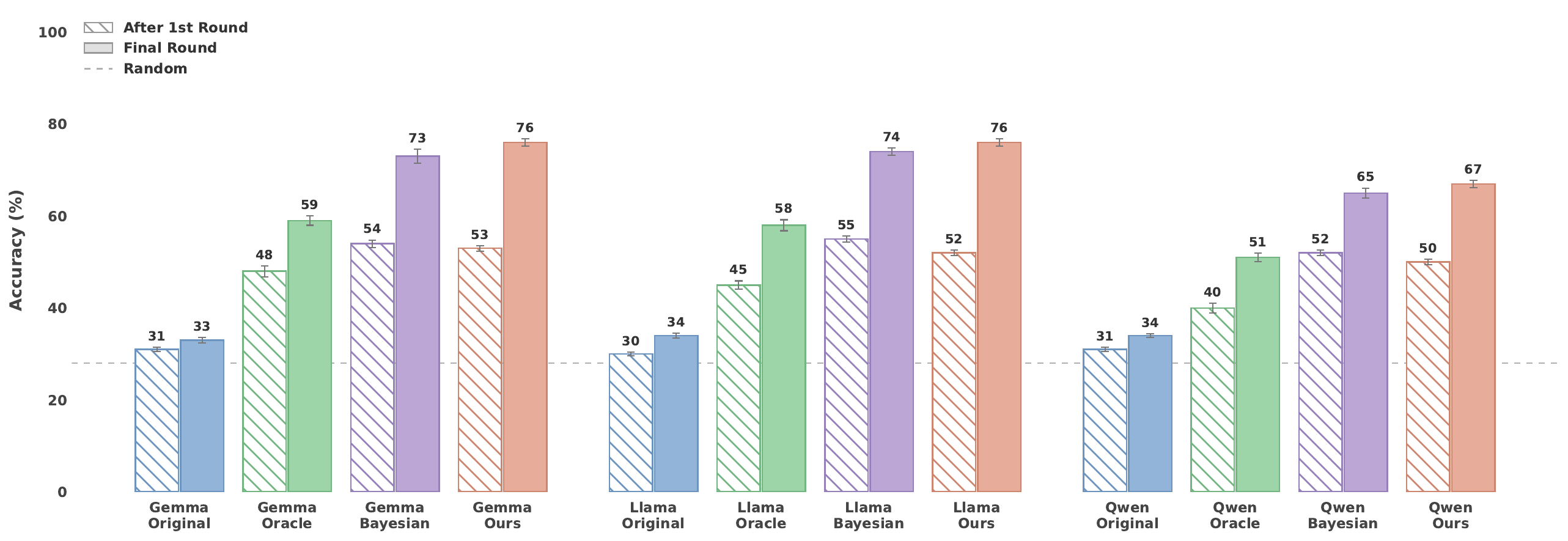}
\caption{\textbf{Interactive preference inference on flight recommendation task.} We show accuracy after the first and final (fifth) rounds across different assistants, including original LLMs, models fine-tuned with the Bayesian Assistant, and models fine-tuned with an oracle that provides correct answers (models provided by the~\cite{qiu2026bayesian}). Both fine-tuning approaches improve performance, and our method achieves the best overall results. Error bars denote the standard error over three random seeds. }
\label{fig:flight_main}
\end{figure*}

\section{Ablation Study Details}
\label{app:ablation}

This appendix details the Flight ablations using Gemma~2 9B.
Final-round results are summarized in the Flight column of
\Cref{tab:cross_domain_ablation}; we provide detailed
descriptions and analysis below.

The symbolic belief tracking module's contribution is demonstrated by
the main results (\Cref{tab:main_results_flight}): Self-consistency
uses the same $N{=}5$ LLM calls but lacks a symbolic posterior,
reaching only 44.7\% versus 76.2\% for \textsc{AdaptFuse}: a
31.5-point gap that isolates the value of exact Bayesian belief
tracking.
The ablation therefore focuses on the remaining design choices.

\subsection{Fusion Mechanism}
\label{app:ablation_fusion}

\paragraph{Fixed fusion ($\lambda = 0.5$).}
This variant replaces the entropy-adaptive weighting
(\Cref{eq:adaptive_weights}) with a fixed equal-weight blend:
$\pi^*(i) \propto 0.5 \cdot \pi^{\mathrm{llm}}_i
+ 0.5 \cdot \pi^{\mathrm{sym}}_i$.
The first-round accuracy (52.4\%) is close to the full system (53.1\%),
which is expected: in early rounds the symbolic posterior is still
diffuse, so the adaptive weights are themselves close to 50/50.
The gap widens in later rounds (72.7\% vs.\ 76.2\%), because fixed
fusion continues to give the LLM half the total weight even after the
symbolic posterior has concentrated and the LLM signal is primarily
adding noise.

\subsection{LLM Aggregation Strategy}
\label{app:ablation_aggregation}

\paragraph{Majority vote.}
Simple majority vote over $N=5$ samples replaces Dirichlet aggregation;
confidence scores $c_s$ are ignored and unchosen options receive zero
mass.
On Gemma Flight, majority vote achieves 50.6\% after the first round
and 73.1\% after the final round.
The 3.1-point drop (73.1\% vs.\ 76.2\%) arises because majority vote
discards confidence information and produces a spikier distribution that
interacts poorly with entropy-adaptive weighting (a spiky but noisy LLM
signal can receive unduly high fusion weight).

\paragraph{Dirichlet without EMA ($m = 0$).}
Each round's LLM distribution is computed solely from the $N=5$ samples
collected at that round, with no temporal smoothing.
The 2.2-point drop (74.0\% vs.\ 76.2\%) shows that EMA smoothing
reduces per-round variance and stabilizes the fusion weights, with the
effect most pronounced in middle rounds where the system still relies
partially on the LLM signal.

\section{AI Use Statement}
\label{app:ai_use}

In this work, we used generative AI tools to generate synthetic datasets,
assist with the design of and provide feedback on research methodology
and experiments, implement methods, assist with translation, clean and
reformat datasets, and edit the manuscript for readability. We also used
these tools to create or modify scientific figures and images and to
create or edit software code. We did not use generative AI tools to
develop theoretical models or conceptual frameworks, formulate
mathematical claims, provide key components of mathematical proofs,
assist with writing proofs, propose or refine hypotheses, conduct
qualitative or thematic data analysis, or interpret results.

We reviewed all AI-assisted work. All LLM-generated code was thoroughly
tested and verified by three authors. All formatted tables in the
manuscript were manually cross-checked against the raw data by two
authors to verify accuracy. We take responsibility for the final content
of this work, including all text, claims, and artifacts produced with
the assistance of generative AI.

\end{document}